\documentclass{article}
\usepackage{ijcai26}

\usepackage{times}
\usepackage{soul}
\usepackage{url}
\usepackage[hidelinks]{hyperref}
\usepackage[utf8]{inputenc}
\usepackage[small]{caption}
\usepackage{graphicx}
\usepackage{amsmath}
\usepackage{amsthm}
\usepackage{amssymb}
\usepackage{enumitem}
\usepackage{booktabs}
\usepackage[switch]{lineno}
\usepackage{tikz}
\usetikzlibrary{positioning,chains,arrows,arrows.meta,shadows,fadings,shapes,backgrounds,matrix,patterns,plotmarks,trees,mindmap,automata,fit,calc,decorations.text}
\usepackage{pgfplots, pgfplotstable}
\usepackage{tabularx}
\usepackage{thmtools}
\usepackage{mathtools}

\usepackage[noend,linesnumbered]{algorithm2e} 
\usepackage{microtype}
\usepackage{multirow, makecell}

\newcommand{\ALCQIfeat}{\ensuremath{\mathcal{ALCQI}(f)}\xspace}
\newcommand{\ALCQfeat}{\ensuremath{\mathcal{ALCQ}(f)}\xspace}

\newcommand{\ALCfeat}{\ensuremath{\mathcal{ALC}(f)}\xspace}
\newcommand{\ALCIfeat}{\ensuremath{\mathcal{ALCI}(f)}\xspace}

\usepackage{misc}
\usepackage{xspace}

\newtheorem{example}{Example}
\newtheorem{theorem}{Theorem}
\newtheorem{lemma}{Lemma}

\newtheorem{definition}{Definition}

\title{Bounded Fitting for Expressive Description Logics}

\author{
Maurice Funk$^1$
\and
Jean Christoph Jung$^2$\and
Tom Voellmer$^{2,3}$
\affiliations
$^1$Leipzig University and ScaDS.AI Center Dresden/Leipzig\\
$^2$TU Dortmund University\\
$^3$Center for Trustworthy Data Science and Security, University Alliance Ruhr\\
\emails
mfunk@informatik.uni-leipzig.de,
\{jean.jung, tom.voellmer\}@tu-dortmund.de,
}

\begin{document}

\maketitle
\begin{abstract}
    Bounded fitting is an attractive paradigm for learning logical formulas from labeled data examples that offers PAC-style generalization guarantees and can often be implemented leveraging SAT solvers. 
    It has been successfully applied to learning concepts of the description logic $\ALC$. We study bounded fitting for learning concepts in expressive description logics that extend $\ALC$ with inverse roles, qualified number restrictions, and feature comparisons. We investigate under which conditions bounded fitting keeps its favorable theoretical properties in this setting, and implement it using a SAT solver. We compare our tool with state-of-the-art concept learners with encouraging results, demonstrating that it is a practical approach to expressive concept learning.
\end{abstract}

\section{Introduction}

Fitting logical formulas to labeled data examples is the basic task of computing, given some data instance $I$, positive examples $P$, and negative examples $N$, a formula $\varphi$ \emph{fitting} $P,N$ in $I$, that is, $\varphi(a)$ holds in $I$ for all $a\in P$ but for no $a\in N$. It underlies many higher level tasks such as example-driven query specification, data exploration, explaining the labeling of the examples, and knowledge acquisition~\cite{DBLP:journals/is/Martins19,DBLP:journals/sigmod/CateFJL23,DBLP:conf/afips/Zloof75}.
A recently prominent approach to finding fitting formulas is \emph{bounded fitting}, which rests on the idea to test whether there are fitting formulas of increasing size~\cite{DBLP:conf/ijcai/CateFJL23,ISWC25,DBLP:conf/fmcad/NeiderG18,DBLP:conf/ijcar/PommelletSS24}. Algorithm~\ref{alg:boundedfitting} illustrates the idea for a general logic \Lmc. Bounded fitting is an attractive paradigm both from a theoretical and a practical perspective. Indeed, it is \emph{complete}, meaning that it will find a fitting formula whenever there is one, and always returns a fitting formula of \emph{minimal size}. The latter is desired by users since they typically prefer small fittings, and implies that bounded fitting is an \emph{Occam algorithm}~\cite{Blumer87}, which in turn guarantees that the found concept generalizes well to unseen examples in the sense of Valiant's probably approximately correct (PAC) learning~\cite{DBLP:journals/cacm/Valiant84,DBLP:journals/jacm/BlumerEHW89}. Finally, bounded fitting can often be implemented with the help of a SAT solver, allowing bounded fitting algorithms to perform well in practice.

In this paper, we are interested in description logic (DL) concepts, an important class of logical formulas used for conceptual modeling, querying knowledge bases, and as a building block of ontologies~\cite{OWL2DLweb}. The importance of DLs in data management has lead to the development of several tools for \emph{DL concept learning}, that is, the task of finding a fitting DL concept given some knowledge base~\cite{DBLP:conf/ijcai/FunkJLPW19,DBLP:journals/ml/LehmannH10}. Existing tools are based on different techniques such as
refinement operators as used in DL-Learner~\cite{DBLP:conf/www/BuhmannLWB18}, evolutionary algorithms as used in EvoLearner~\cite{DBLP:conf/www/HeindorfBDWGDN22}, or decision trees as in TDL~\cite{DBLP:conf/pkdd/DemirYRMN25}.
Recently, bounded fitting was demonstrated to be a suitable alternative for learning concepts formulated in the standard DLs \EL and \ALC~\cite{DBLP:conf/ijcai/CateFJL23,ISWC25}.

\begin{algorithm}[t]
\caption{Bounded Fitting for logic \Lmc.}\label{alg:boundedfitting}
   \KwIn{Data instance $I$, examples $P,N$}
   
   \For {$k:=1,2,\ldots$}{
        \If{there is $\varphi\in \Lmc$ of size $k$ that fits $P,N$ in $I$}{\Return $\varphi$}
   }
\end{algorithm}
However, these DLs do not allow to express essential properties present in  real-world data such as
inverse roles, counting, and numeric features. Hence, in this paper we consider the DL $\ALCQIfeat$, which is the target language in EvoLearner and DL-Learner and which extends $\ALC$ with qualified number restrictions (abbreviated with $\Qmc$), inverse roles ($\Imc$), and feature comparisons ($f$), a simple mean to access data properties. One typical \ALCQIfeat concept is
\[\text{Elephant}\sqcap \text{Male} \sqcap ((\text{weight}\geq 3t)\sqcup (\geq 2\text{ childOf}^-.\text{Female}))\]
expressing the concept of male elephants that weigh at least three tons or have at least two daughters. 
The additional expressiveness over $\ALC$ allows to learn more interesting and meaningful concepts.

We make three main contributions. 
First, we study the fitting problem for $\ALCQIfeat$ and show that one can compute in polynomial time a concept that fits a maximal number of the given examples. The algorithms follow standard ideas based on computing maximal bisimulations, but the fact that one can efficiently compute optimal fitting concepts was surprising to us. While such optimal fitting algorithms are good news from a pure fitting perspective, under standard complexity assumptions they cannot be  Occam algorithms or have PAC-style generalization guarantees.

We thus investigate as our second contribution how to extend the bounded fitting paradigm to $\ALCQIfeat$. 
To this end, we show first that the \emph{size-restricted fitting problem} that has to be solved in Line~2 of Algorithm~\ref{alg:boundedfitting} is \NPclass-complete for \ALCQIfeat, as it is for $\ALC$ and many of its fragments~\cite{ISWC25}. This paves the way to the use of SAT solvers. We then observe that inverse roles can be reduced away in a simple fashion. In contrast, adapting bounded fitting to number restrictions and feature comparisons while having a realistic implementation based on SAT solvers in mind is more challenging. We resort to allow for number restrictions and feature comparisons in a controlled and incremental way. While this is in spirit of \emph{Occam's razor}, which demands to prefer simpler concepts, and can be implemented using a SAT solver, it is not immediately clear that it results in Occam algorithms as formalized by Blumer \textit{et al.}~(\citeyear{DBLP:journals/jacm/BlumerEHW89}). Our main results are then sufficient conditions for when the constructed algorithms are indeed Occam, (and thus have PAC-style guarantees.

As our third contribution, we implement bounded fitting for $\ALCQIfeat$, using a SAT solver to solve the size-restricted fitting problem.
In order to achieve competitive performance, we describe two techniques that speed up SAT-based bounded fitting. First, we propose that bisimilarity can be exploited as a preprocessing step to reduce the size of the input database. This is not specific to bounded fitting and can potentially be used to improve other concept learning systems as well. Second, we observe that bounded fitting allows easy parallelization and implement it to use multiple CPU cores.

We evaluate our implementation on several benchmarks and compare it to the state-of-the-art concept learners EvoLearner and TDL. The experiments show that the concepts learned by our implementation generalize at least as good as the ones found by EvoLearner and TDL on the standard structured machine learning (SML) benchmarks. Moreover, using a newly generated benchmark, we show that our tool is better in finding fitting $\ALCQ$ concepts than other systems. Finally, we confirm that the presented optimization are effective in practice as well. Overall, our evaluation demonstrates that bounded fitting works well in practice also for expressive DLs.

Missing proofs can be found in the appendix.

\paragraph{Related Work}
The works most relevant related to this paper are~\cite{DBLP:conf/ijcai/CateFJL23} and~\cite{ISWC25}, which introduce the bounded fitting framework for the DLs $\EL$ and $\ALC$, respectively. Beyond the concept learners mentioned above, there is a series of other tools that rely on more classical machine learning techniques such as trained search heuristics and neural concept synthesis like ROCES~\cite{DBLP:conf/ijcai/KouagouHDN24}, NCES/NCES2~\cite{DBLP:conf/pkdd/KouagouHDN23}, and Drill~ \cite{DBLP:conf/ijcai/DemirN23}. As they require a separate pretraining step, it is difficult to make a meaningful comparison with them. Other concept learners are DL-Foil~\cite{DBLP:journals/fgcs/RizzoFd20} and SPaCel~\cite{DBLP:journals/jmlr/TranDGM17}.

\section{Background}\label{sec:prelims}

We introduce syntax and semantics of the description
logic~\ALCQIfeat, the extension of \ALCQI~\cite{DL-Textbook} with simple feature comparisons. Let \NC, \NR, \NF, and \NI be mutually disjoint and countably infinite sets
of \emph{concept names}, \emph{role names}, \emph{feature names}, and \emph{individual names}, respectively. We assume that each feature name $f$ has a \emph{domain} $\dom(f)$, the set of values the feature $f$ can take. We further assume a total order $\leq$ in every domain.
An \emph{\ALCQIfeat concept} $C$ is defined according to the
syntax rule
\[
C, D ::= \top \mid A \mid \neg C \mid C \sqcap D \mid (\bowtie n\ R.C) \mid (f\bowtie v)
\]
where $A$ ranges over concept names, ${\bowtie}\in\{\leq,\geq\}$, $R$ is a role name $r$ or an \emph{inverse role} $R=r^-$, $n \geq 0$, $f\in\NF$ is a feature name, and $v\in\dom(f)$. We call concepts of the shape $(\bowtie n\ R.C)$ \emph{qualified number restrictions} and concepts of shape $(f\bowtie v)$ \emph{feature comparisons}.
We use the standard abbreviations $\bot$, $C\sqcup D$, $\exists R.C$, and $\forall R.C$ for $\neg \top$, $\neg (\neg C\sqcap \neg D)$, $(\geq 1\ R.C)$, and $(\leq 0\ R.\neg C)$, respectively. We also consider the usual restrictions of $\ALCQIfeat$ that are obtained by disallowing certain constructors in the syntax. For example, $\ALCQ$ is obtained by disallowing inverse roles $r^-$ and feature comparison, while $\ALCIfeat$ is obtained by disallowing qualified number restrictions.

A \emph{database} $I$ is a finite set of facts of the form $A(a)$, $r(a,b)$, and $f(a,v)$ for $A\in\NC$, $r\in\NR$, $f\in\NF$, $a,b\in\NI$, and $v\in\dom(f)$. We let $\adom(I)$ denote the set of all individual names that occur in $I$. Moreover, we assume that for every $a\in\adom(I)$ and $f\in\NF$, $I$ contains at most one fact of shape $f(a,v)$. We write $r^-(b,a)\in I$ if $r(a,b)\in I$, and define $\mn{succ}_I(a,R)$ as the set \[\mn{succ}_I^R(a)=\{b\mid R(a,b)\in I\}.\]

The semantics $C^I$ of an \ALCQIfeat concept $C$ in database $I$ is defined inductively as follows. In the inductive base, we set $\top^I=\adom(I)$ and $A^I=\{a\mid A(a)\in I\}$. In the inductive step, we define $(\neg C)^I=\adom(I)\setminus C^I$, $(C\sqcap D)^I = C^I\cap D^I$, and
\begin{align*}
    (\bowtie n\ R.C)^I & = \{a\in\adom(I)\mid |\mn{succ}^R_I(a)\cap C^I|\bowtie n\}, \\
    (f\bowtie v)^I & = \{a\in \adom(I)\mid f(a,v')\in I\text{ and }v'\bowtie v\}.
\end{align*}

A \emph{signature} is a finite set $\Sigma$ of concept, role, and feature names. A $\Sigma$-database is a database in which all facts mention only concept, role, and feature names from $\Sigma$. For a syntactic object $O$ such as a concept or a database, we let $\lVert O \rVert$ denote the \emph{size} of $O$, that is, the length of a string representation of $O$ in a suitable fixed alphabet in which concept/role/feature names contribute $1$. We assume unary coding of numbers in qualified number restrictions. 

\paragraph{Fitting Problems}
Let $I$ be a database and $P,N\subseteq\adom(I)$ be sets of individuals called \emph{positive} and \emph{negative examples}.
We say that an \ALCQIfeat concept $C$ \emph{fits $P,N$ in $I$} if $a\in C^{I}$ for each $a\in P$ and $b\notin C^{I}$ for each $b\in N$. The \emph{fitting problem} for \ALCQIfeat is to decide, given $I,P,N$, whether there is an \ALCQIfeat concept that fits $P,N$ in $I$. A \emph{fitting algorithm for \ALCQIfeat} takes as input $I, P, N$ as above, and returns an \ALCQIfeat concept that fits $P, N$ in $I$, if it exists. The fitting problem and fitting algorithms for fragments of \ALCQIfeat are defined accordingly. Observe that we make the \emph{closed world assumption} in contrast to other works which study the fitting problem for knowledge bases, see, e.g.~\cite{DBLP:conf/ijcai/FunkJLPW19,DBLP:journals/ai/JungLPW22}).

\begin{example} Consider the following database $I$:

     \begin{tikzpicture}[on grid]
            \node[] (a) at (0, 0) {$a$};
            \node[below left = 0.9cm and 1.1cm of a] (a0) {$a_1$};
            \node[below = 0.8cm of a0] (w0) {$121cm$};
            \node[below right = 0.9cm and 1.1cm of a] (a1) {$a_2$};
            \node[below = 0.8cm of a1] (w1) {$145cm$};

            \node[ ] (b) at (4, 0) {$b$};
            \node[below left = 0.9cm and 1.1cm of b] (b0)  {$b_1$};
            \node[below = 0.8cm of b0] (w2) {$152cm$};

            \node[below right = 0.9cm and 1.1cm of b] (b1)  {$b_2$};
            \node[below = 0.8cm of b1] (w3) {$163cm$};

            \draw[-Latex] (a) edge node[above, sloped] {\footnotesize child} (a1);
            \draw[-Latex] (a) edge node[above, sloped] {\footnotesize child} (a0);
            \draw[-Latex] (b) edge node[above, sloped] {\footnotesize child} (b0);
            \draw[-Latex] (b) edge node[above, sloped] {\footnotesize child} (b1);
            \draw[-Latex] (a0) edge node[right] {\footnotesize height} (w0);
            \draw[-Latex] (a1) edge node[right] {\footnotesize height} (w1);
            \draw[-Latex] (b0) edge node[right] {\footnotesize height} (w2);

            \draw[-Latex] (b1) edge node[right] {\footnotesize height} (w3);
    \end{tikzpicture}
    The $\ALCQIfeat$ concept $(\leq 1\ \mathit{child}.(\mathit{height} \geq 140cm))$ fits $\{a\}, \{b\}$ in $I$. Further, there is no $\ALCQI$ concept that fits $\{a\}, \{b\}$.
\end{example}

\paragraph{Occam Algorithms} We recall the standard notion of  \emph{Occam algorithms} following~\cite{DBLP:journals/jacm/BlumerEHW89}.
Let $\Cmc$ be a set of DL concepts, $I$ be a database and $S\subseteq \adom(I)$ be a set of examples. We say that $\Cmc$ \emph{shatters $S$ over $I$} if for every $S' \subseteq S$ there is a $C \in \Cmc$ such that $S' = S \cap C^I$.
The \emph{VC-Dimension} of a set of concepts $\Cmc$ is the largest number $d \geq 0$ such that there is a database $I$ and a set of examples $S$ with cardinality $d$ that is shattered by $\Cmc$.

Let $\mathbf{A}$ be a fitting algorithm. For a signature $\Sigma$ and $s, m \geq 1$, let $\mathcal H^{\mathbf A}(\Sigma,s,m)$ be the set of all concepts returned by $\mathbf A$ when started on some input $I, P, N$ such that $I$ is a $\Sigma$-database, $|P| + |N| = m$, and there is a concept $C$ with $\lVert C \rVert \leq s$ that fits $P, N$ in $I$. We call $\mathbf{A}$ an \emph{Occam algorithm} if there is a polynomial $p$ and $\alpha\in[0,1)$ 
such that for all $\Sigma$ and $s,m \ge 1$ the VC-Dimension of $\mathcal H^{\mathbf A}(\Sigma,s,m)$ is at most $p(s,|\Sigma|)\cdot m^\alpha$.

It is well known that that every Occam algorithm has PAC-style generalization meaning that it is a \emph{sample-efficient PAC learning algorithm}. As the precise definition of that notion is not important for the main part of the paper, we refer the reader to~\cite{DBLP:journals/jacm/BlumerEHW89} or to the supplementary material. 

\section{The Fitting Problem for \ALCQIfeat}\label{sec:fitting}

We begin our investigation by showing that the fitting problem for \ALCQIfeat and its fragments is efficiently solvable.

\begin{theorem}\label{thm:fitting}
    The fitting problems for \ALCQIfeat, \ALCQfeat, \ALCQI, and \ALCQ are decidable in polynomial time. 
\end{theorem}
This can be shown using standard arguments. Since some of these will be important later, we sketch the main steps here; for details see the supplementary material. We first show the theorem for \ALCQ and then for the remaining languages by a reduction to \ALCQ. We rely on the notion of bisimulations.

\begin{definition}
    Let $I$ be a database. An \emph{\ALCQ bisimulation on $I$} is a relation $S \subseteq \adom(I) \times \adom(I)$ such that the following conditions are satisfied for all $(a, b) \in S$:

    \begin{enumerate}
        \item for all $A \in \NC$, $A(a) \in I$ if and only if $A(b) \in I$;
        
        \item for all $r \in \NR$ and $D \subseteq \mn{succ}_I^r(a)$,
        there is $E \subseteq \mn{succ}_I^r(b)$ such that $S$ contains a bijection between
        $D$ and $E$;
        \item for all $r \in \NR$ and $E \subseteq \mn{succ}^r_I(b)$,
        there is $D \subseteq \mn{succ}^r_I(a)$ such that $S$ contains a bijection between
        $D$ and $E$.
    \end{enumerate}

    We call $a,b\in \adom(I)$ \emph{\ALCQ bisimilar}, written $I,a \sim I,b$, if there is an $\ALCQ$ bisimulation $S$ on $I$
    with $(a, b) \in S$.
\end{definition}

It is well-known that the existence of fitting $\ALCQ$ concepts is characterized by the non-existence of \ALCQ bisimulations between positive and negative examples. More precisely, there is an \ALCQ concept fitting $P,N$ in $I$ iff $I,a\not\sim I,b$ for all $a\in P,b\in N$~\cite{DBLP:journals/sLogica/Rijke00}. Since bisimilarity can be decided in polynomial time, Theorem~1 for \ALCQ follows.

It remains to provide the promised reductions. 

\begin{restatable}{lemma}{lemreduction}\label{lem:reduction}
  There are polynomial-time reductions from the fitting problem for \ALCQIfeat to that for \ALCQfeat, and from the fitting problem for \ALCQfeat to that of~\ALCQ.
\end{restatable}

\begin{proof}
For the first reduction, let $I,P,N$ be an input to the fitting problem for \ALCQIfeat. Introduce a fresh role name $\overline r$ for every role name $r$ that occurs in $I$. Then then there is an \ALCQIfeat concept fitting $P,N$ in $I$ iff there is an \ALCQfeat concept fitting $P,N$ in $J=I\cup \{\overline r(b,a)\mid r(a,b)\in I\}$. 

For the second reduction, let $I,P,N$ be an input to the fitting problem for \ALCQfeat. 
Introduce a fresh concept name $A_{f \geq v}$ for every feature $f$ and value $v$ such that there is a fact $f(a,v)\in I$. Let $J$ be the database consisting of all facts of shape $A(a)$ and $r(a,b)$ from $I$ and 
all facts $A_{f \geq v}(a)$ such that there is a fact $f(a, v') \in I$ with $v \geq v'$. It is routine to verify that 
there is an \ALCQfeat concept fitting $P,N$ in $I$ if and only if there is an \ALCQ concept fitting $P,N$ in~$J$. 
\end{proof}

We show next that the algorithms used to prove Theorem~\ref{thm:fitting} can be made constructive in a strong sense. Note that we allow DAG representation of concepts when constructing them. 
\begin{restatable}{theorem}{thmoptimal}\label{thm:optimal}
    There are polynomial time fitting algorithms for \ALCQIfeat, \ALCQfeat, \ALCQI, \ALCQ which always compute a concept that fits a maximal overall number of positive and negative examples.
\end{restatable}
While such optimal fitting algorithms are natural candidates to be used in practice, they cannot be sample-efficient PAC learning algorithms, under standard complexity assumptions. 
Indeed, if a polynomial time fitting algorithm for \ALCQIfeat was a sample-efficient PAC learning algorithm, it could be used as an efficient PAC learning algorithms for propositional logic which likely does not exist~\cite{DBLP:journals/jacm/KearnsV94}. 
Hence, the algorithms in Theorem~\ref{thm:optimal} are also not Occam.

\section{Bounded Fitting for \ALCQIfeat}\label{sec:boundedfitting}

Motivated by the fact that standard algorithms are not Occam or sample-efficient, we apply the bounded fitting paradigm given in Algorithm~\ref{alg:boundedfitting} to \ALCQIfeat, and investigate under which conditions the favorable properties are preserved. As we are interested in realistic implementations leveraging SAT solvers, we first establish that the core problem to be solved in Line~2 of Algorithm~\ref{alg:boundedfitting} is indeed \NPclass-complete for \ALCQIfeat.  

Formally, the \emph{size-restricted fitting problem} is defined as follows. Given database~$I$, positive and negative examples $P,N$ and $k\geq 0$, decide whether there is an $\ALCQIfeat$ concept $C$ of size $\lVert C \rVert\leq k$ that fits $P,N$ in $I$.
In the context of bounded fitting, it is natural to assume $k$ to be coded in unary.

\begin{restatable}{theorem}{thmsizefittingnpcomplete}\label{thm:size-fitting-np-complete}
      The size-restricted fitting problem for \ALCQIfeat is \NPclass-complete. Hardness already holds for inputs $I,P,N$ where $I$ is an $\{A,r,s\}$-database for $A\in\NC$ and $r,s\in \NR$, and $P,N$ are sets with $|P|=|N|=1$.
\end{restatable}
The upper bound is established by the usual \emph{guess-and-check} argument: given $I,P,N$, non-deterministically guess a candidate concept of size $k$, and deterministically verify that it fits $P, N$ in $I$ in time polynomial in $k$ and the size of $I$. The proof of the lower bound is more challenging. It extends the proof of the same result for \ALC~\cite{ISWC25}, but requires care to handle inverse roles and number restrictions.

While the \NPclass-upper bound from Theorem~\ref{thm:size-fitting-np-complete} is encouraging from the theoretical perspective, a direct reduction to SAT in the style of Cook's theorem~\cite{Cook71} is hardly practical. Previous reductions were based on the idea of measuring the complexity of the concept sought in stage $k$ of bounded fitting in terms of the number of nodes in the syntax tree or, equivalently, the number of logical operators~\cite{DBLP:conf/ijcai/CateFJL23,ISWC25}. The SAT solver then ``guesses'' the labels of the nodes in the syntax tree and evaluates whether the represented concept fits. This idea does not easily to \ALCQIfeat since the numbers that occur in qualified number restrictions and feature comparisons should in some way contribute to the complexity of the concepts. We discuss below the additions $\Imc,\Qmc,f$ to \ALC separately for the sake of clarity. Our bounded fitting algorithm for \ALCQIfeat combines all ideas.

We start with inverse roles. Observe that the first reduction in Lemma~\ref{lem:reduction} provides a transparent way of dealing with them. Consider the following algorithm $\mathbf A_\Imc$. Let $I,P,N$ be an input for the fitting problem for \ALCI. Then, $\mathbf A_\Imc$ computes $J,P,N$ for $J$ defined as in the proof of Lemma~\ref{lem:reduction}, and starts bounded fitting (for \ALC) on input $J,P,N$. If it finds an \ALC concept in this way, it outputs the corresponding \ALCI concept. It follows from the reduction and the fact that bounded fitting for \ALC is a complete Occam algorithm that $\mathbf A_\Imc$ is a complete Occam algorithm as well.

\begin{algorithm}[t]
\caption{Bounded fitting for \ALCQ.}\label{alg:boundedfittingalcq}
   \KwIn{Database $I$, examples $P,N$}
   
   \For {$k:=1,2,\ldots$}{
        \If{there is $C\in \ALCQ_{g(k)}^k$ that fits $P,N$ in $I$}{\Return $C$}
   }
\end{algorithm}

We next consider qualified number restrictions. Note that under the discussed measure of size (being the number of logical operators), there are, in principle, infinitely many concepts of a given size. To make the space of concepts considered in each stage of bounded fitting finite, we have to bound the numbers occurring in the concepts. 
To this end, we allow in stage $k$ numbers up to $g(k)$ in concepts, for some function~$g$. The idea is illustrated in Algorithm~\ref{alg:boundedfittingalcq}, where $\ALCQ^k_{g(k)}$ denotes the set of \ALCQ concepts that allow for at most $k$ logical operators and numbers in number restrictions at most $g(k)$. This approach is in the spirit of Occam's razor if $g$ is an increasing function as this reflects that larger numbers in number restrictions are considered more complex. Depending on the domain, the user may choose different functions reflecting their view on the importance of number restrictions. The following theorem provides the user with a range of functions $g$ which make Algorithm~\ref{alg:boundedfittingalcq} a complete Occam algorithm. 
\begin{restatable}{theorem}{thmalcqoccam}\label{thm:alcq-occam}
Let $g\in\Omega(k)\cap O(2^{p(k)})$ for some polynomial $p$. Then Algorithm~\ref{alg:boundedfittingalcq} is a complete Occam algorithm.
\end{restatable}

Completeness follows from $g\in\Omega(k)$. For showing that Algorithm~2 is Occam we carefully analyze its space $\Hmc^{\Abf}(\Sigma,s,m)$ and show that its VC-dimension is polynomially bounded for any $g$ as in the theorem. 

We next look at feature comparisons. Similar to the number restrictions, there are potentially infinitely many concepts with a given number of logical operators, since features may have an infinite domain. Similar to what was done for \ALCQ, we could obtain an Occam algorithm by allowing in every stage $k$ of bounded fitting for, say, $k$ bits in the representation of the feature values. From a practical perspective, however, we believe that different feature values should contribute in the same way to a learned concept, e.g., the complexity of feature comparisons $(\text{salary}\geq n)$ should not depend on $n$. This perspective suggests to remove the feature comparisons via the reduction given in Lemma~\ref{lem:reduction}, similar to what we did for inverses. Unfortunately, this does not lead to an Occam algorithm. Let 
$\mathbf A_f$ be the algorithm obtained by first applying the reduction of fitting in \ALCfeat to fitting in \ALC from Lemma~\ref{lem:reduction} and then applying bounded fitting for \ALC to the result. Clearly, $\mathbf A_f$ is complete, but:  
\begin{restatable}{theorem}{thmreductionnotoccam}\label{thm:reduction-not-occam}
    Algorithm $\mathbf{A}_f$ is not an Occam algorithm.
\end{restatable}
The intuitive reason for this is that for sufficiently large $s,m$, one can find inputs to the algorithm so that $\mathbf A_f$ returns concepts of the form $\exists r.((f\le j)\sqcap (f\ge j))$, for any $j\in\mathbb{N}$. While an infinite space $\Hmc^{\Abf_f}(\Sigma,s,m)$ does not rule out being an Occam algorithm, we show that these concepts shatter exponentially sized sets of examples, making the VC dimension not polynomially bounded. Indeed, we show the stronger statement that $\mathbf{A}_f$ is not a sample-efficient PAC learning algorithm.

The proof of Theorem~\ref{thm:reduction-not-occam} relies on the presence of an unbounded number of feature values and on databases with unbounded outdegree. We next show that both these assumptions are necessary to achieve exponential VC-dimension. A class of databases has \emph{bounded outdegree} if there is a constant $\ell$ such that for all $a\in\adom(I)$, there are at most $\ell$ facts of shape $r(a,b)\in I$. Similarly, the class of databases has \emph{bounded domains} if the cardinalities of all feature domains are bounded by some constant. 

\begin{restatable}{theorem}{thmoccamconstantoutdegree}\label{thm:reduction-constantdegree-occam}
    Algorithm $\mathbf{A}_{f}$ is an Occam Algorithm if input databases have bounded outdegree or bounded domains.  
\end{restatable}
If domains are bounded, one can replace features with a fixed number of concept names. For bounded outdegree, we observe that for a concept $C$ of a given size $s$, there is an at most exponential number of feature values ``visible'' by $C$ in any given set of $m$ examples. This allows us to bound the VC-dimension of the space $\mathcal H^{\mathbf A_{f}}(\Sigma,s,m)$ over databases of bounded outdegree.

From a practical point of view, the preconditions  in Theorem~\ref{thm:reduction-constantdegree-occam} are not guaranteed to hold in all datasets, but in relevant cases, e.g., for features that describe days of the week, age of employees, or for roles such as child or spouse. A limitation in practice, however, is that using this route introduces a potentially large number of concept names which, as we shall see, is reflected in an increased number of propositional variables. We describe in Section~5.2 how we address this problem in our tool.

\section{Implementation}

We implemented bounded fitting for $\ALCQIfeat$. Our implementation follows the same general approach as ALCSAT~\cite{ISWC25}: in each stage of bounded fitting, the size-restricted fitting problem is solved by checking satisfiability of a suitable propositional formula, using a SAT solver. Users may specify the fragment of \ALCQIfeat that they are interested in. Moreover, our tool implements an approximation strategy that tries to compute a concept that fits the maximum number of examples. In the remainder of the section, we describe the reduction to SAT for \ALCQ, discuss how we handle inverses (\Imc) and features ($f$), and describe two optimizations. 

\subsection{SAT Encoding for $\ALCQ$}

\newcommand{\sigr}{\ensuremath{\Sigma_{\textsf{R}}}\xspace}
\newcommand{\sigc}{\ensuremath{\Sigma_{\textsf{C}}}\xspace}

We briefly recall some details of ALCSAT's implementation~\cite{ISWC25} in order to explain how
we extend it to $\ALCQ$. Let $I, P, N, k$ be an instance of the size-restricted fitting problem for $\ALC$.
ALCSAT constructs a propositional
formula $\varphi$ that is satisfiable if and only if there exists an $\ALC$ concept
$C$ of size $k$ that fits $P$, $N$ in $I$. Moreover, such a fitting concept can directly be obtained from a model of $\varphi$.
Let $\sigc$ and $\sigr$ be the sets of concept and role names that
occur in $I$. 
The formula $\varphi$ encodes the syntax tree of an $\ALC$ concept with $k$ nodes and its semantics such that in a model of $\varphi$,
for all $i, j \in \{1, \ldots, k\}$, $a \in \adom(I)$ and $v$ a node label from
$\Vmc = \{\top,\bot,\neg,\sqcap,\sqcup\} \cup \sigc \cup \{\exists r, \forall
r\mid r\in \sigr\}$,
the variable
$x_{i, v}$ is true iff node $i$ of the syntax tree is labeled with $v$, 
the variable $y_{1,i,j}$ is true iff node $i$ has the single successor $j$, 
the variable $y_{2, i, j}$ is true iff node $i$ has the two successors $j$ and $j + 1$, and
the variable $z_{i, a}$ is true iff $a \in C_i^I$ where $C_i$ is the concept
represented by node $i$ in the syntax tree.
For example, the semantics of existential restrictions is encoded by clauses that express 
\[
x_{i, \exists r} \land y_{1, i, j} \land z_{i, a} \to \bigvee \{z_{j, b} \mid r(a, b) \in I \}
\]
for $1 \leq i, j \leq k$, $r \in \Sigma_r$, $a \in \adom(I)$.
The requirement that the concept $C_1$ fits $P,N$ in $I$ is then encoded using $\bigwedge_{a \in P} z_{1, a}$ and $\bigwedge_{a \in N} \neg z_{1, a}$.

Our implementation extends this encoding by allowing additional vertex labels for qualified number restrictions and by encoding their semantics.
Given a bound $g(k)$ (recall Algorithm~\ref{alg:boundedfittingalcq} from the previous section) for numbers in qualified number restrictions, we use node labels
$(\leq n\ r)$ and $(\geq n\ r)$  for all $r \in \Sigma_r$ and $n$ with $0 \leq n \leq g(k)$.
Then, we encode the semantics of qualified number restrictions by adding clauses that encode
\begin{align*}
    x_{i, (\geq n\ r)} \land y_{1, i, j} \land z_{i, a} &\to ( |\{ z_{j, b} \mid r(a, b) \in I \}| \geq n) \\
    x_{i, (\geq n\ r)} \land y_{1, i, j} \land \neg z_{i, a} &\to ( |\{ z_{j, b} \mid r(a, b) \in I \} |\leq n - 1)
\end{align*}
for all $i, j$ with $1 \leq i, j \leq k$, $0 \leq n \leq g(k)$, $r \in \Sigma_r$ and $a \in \Delta^I$ to $\varphi$.
To encode the cardinality constraints in the above formulas as clauses,
we employ a sequential counter encoding, which for a bound of $k$ on $n$ literals uses $O(n \cdot k)$ additional variables and clauses~\cite{DBLP:conf/cp/Sinz05}.
In total, this results in $O(k^2 \cdot |I| \cdot |\Sigma| \cdot g(k)^2 )$ 
additional clauses.

\subsection{Handling Inverses and Features}

We treat inverses as described in the previous section, that is, via the (first) reduction in Lemma~\ref{lem:reduction}. It is worth noting that EvoLearner implements the same strategy~\cite{DBLP:conf/www/HeindorfBDWGDN22}.

To handle features, we implement a variant of the (second) reduction in Lemma~\ref{lem:reduction}. Instead of introducing a single concept name $A_{f\geq v}$ for \emph{every value} $v$ that occurs in $I$, we introduce such concept names only for a \emph{restricted number} $n_f$ of these values. This alone would result in an incomplete algorithm (as we may miss fitting concepts), so we incrementally increase the number $n_f$ along with the stages of bounded fitting. In practice we split the $\dom(f)$ into $n_f$ equally sized intervals. 

DL-Learner follows a similar approach~\cite{DBLP:series/ssw/Lehmann10}. EvoLearner extends that by an entropy-based selection of the values given at the examples $P,N$~\cite{DBLP:conf/www/HeindorfBDWGDN22}.

\subsection{Optimizations}

We present two optimizations. The first relies on \ALCQ bisimulations and is not specific to bounded fitting; it can in principle be used with all other tools. The idea is to exploit that $\ALCQ$ concepts cannot distinguish between $\ALCQ$ bisimilar individuals, and hence to replace the input database $I$ with a smaller database $I'$ in which few elements are bisimilar; we need to keep copies of bisimilar elements, due to the qualified number restrictions. 

For a database $I$ and for every $a \in \adom(I)$, let $[a]$ denote the $\sim$-equivalence class of $a$ in $I$.
The \emph{$\sim$-quotient} of $I$ is the database $I'$ over individuals
$\langle [a], i \rangle $ with $a \in \adom(I)$ and 
\[1 \leq i \leq \max_{b \in \adom(I), r \in \NR} | \mn{succ}^r_I(b) \cap [a] |, \] 
and the following facts for all $A \in \NC$ and $r \in \NR$:
\begin{align*}
    A(\langle [a], i \rangle) & \text{ if } A(a) \in I\\
     r(\langle [a], i \rangle, \langle [b], j \rangle) & \text{ if } j \leq | \mn{succ}^r_I(a) \cap [b] |.
\end{align*}

Using the observation that $I, a$ is $\ALCQ$ bisimilar to $I', \langle a, i\rangle$ for any $i$, one can show the following.
\begin{lemma}\label{lem:bisim-reduction-preserves-concepts}
    Let $I, P, N$ be a fitting problem instance for $\ALCQ$. Every $\ALCQ$ concept $C$
    fits $P, N$ in $I$ if and only if $C$ fits
    $\{ \langle [p], 1 \rangle \mid p \in P \}, \{ \langle [n], 1 \rangle \mid n \in N\}$ in $I'$.
\end{lemma}

Hence, on input $I, P, N$ our implementation computes $I'$, and then runs bounded fitting on $I'$. To compute the \ALCQ bisimulation equivalence classes in $I$, we use a variant of the color refinement algorithm
(see e.g.~\cite{GroheColor2021}), 
which can be implemented to run in $O(|I| \log |I|)$.

Our second optimization concerns parallelization. By design, bounded fitting is suitable for parallelization: we can simply start each stage in its own thread. We follow a slightly different route that was proposed in the context of fitting LTL formulas~\cite{DBLP:conf/fdl/Riener19}. We split the search space at a given stage $k$ into several pieces and start search in each subspace on a different core based on the topology of the concepts. For this purpose we enumerate (outside the SAT solver) all possible topologies of syntax trees, split them into the number of available threads and start bounded fitting in each thread with the additional constraint to consider only the given topologies. 

\medskip

\section{Experiments}

We evaluate different aspects of our tool on the SML-Benchmarks~\cite{DBLP:journals/semweb/WestphalBBJL19} and on benchmarks for $\ALCQ$ fitting we generated from the YAGO 4.5 knowledge base~\cite{DBLP:conf/sigir/SuchanekABCPS24}. All experiments were executed on a MacBook Pro with M3 Pro chip and 36GB RAM.\footnote{Our implementation, the benchmarks, and instructions to reproduce our experiments are available at \url{https://github.com/SAT-based-Concept-Learning/ALCSAT}.}

We refer to accuracy and F1-score of concepts $C$ on examples $P, N$ in $I$. With accuracy we mean the percentage of elements of $P$ and $N$ that are labeled correctly by $C$, formally $\frac{\lvert C^I \cap P \rvert + \lvert N \setminus C^I\rvert}{\lvert P \rvert + \lvert N \rvert}$. As usual, with F1-score we mean the harmonic mean of precision and recall of $C$. 

\subsection{Evaluating Generalization}
\begin{table*}[t!]
\footnotesize
\begin{tabular}{lcccccc}
\toprule
&  Carcinogenesis &  Hepatitis &  Lymphography &  Mammographic & Premierleague & Pyrimidine\\ \midrule
\multirow{2}{*}{EvoLearner} &
$5.50 \pm  1.51$ & 
$4.50 \pm  1.35$ & 
$17.40 \pm  9.03$ & 
$3.20 \pm  1.48$ & 
$5.00 \pm  0.94$ & 
$6.40 \pm  1.26$ 
\\
& $0.71 \pm  0.01$ & 
$0.77 \pm  0.07$ & 
$0.84 \pm  0.08$ & 
$0.78 \pm  0.05$ & 
$1.00 \pm  0.00$ & 
$0.91 \pm  0.14$ 
\\
\midrule
\multirow{2}{*}{TDL} &
$879.60 \pm  132.97$ & 
$1436.80 \pm  46.49$ & 
$112.10 \pm  23.33$ & 
$5780.70 \pm  302.98$ & 
$298.40 \pm  35.22$ & 
$15.90 \pm  3.78$ 
\\
&$0.62 \pm  0.07$ & 
$0.83 \pm  0.05$ & 
$0.82 \pm  0.11$ & 
$0.69 \pm  0.03$ & 
$0.47 \pm  0.37$ & 
$0.82 \pm  0.24$ 
\\
\midrule
\multirow{2}{*}{Theorem \ref{thm:optimal}} & 
$4715.60 \pm  455.32$ & 
$53269.80 \pm  1549.82$ & 
$330.80 \pm  18.98$ & 
$1925.40 \pm  72.76$ & 
$23.40 \pm  1.26$ & 
$60.80 \pm  3.94$ 
\\
&$0.56 \pm  0.06$ & 
$0.53 \pm  0.06$ & 
$0.74 \pm  0.13$ & 
$0.76 \pm  0.05$ & 
$1.00 \pm  0.00$ & 
$0.90 \pm  0.16$ 
\\\midrule
\multirow{2}{*}{Our Tool} & 
$6.00 \pm  0.00$ & 
$5.00 \pm  0.00$ & 
$10.00 \pm  0.00$ & 
$9.10 \pm  0.32$ & 
$2.00 \pm  0.00$ & 
$5.90 \pm  0.32$ 
\\
&$0.73 \pm  0.10$ & 
$0.81 \pm  0.03$ & 
$0.83 \pm  0.08$ & 
$0.81 \pm  0.05$ & 
$0.97 \pm  0.05$ & 
$0.95 \pm  0.12$ 
\\

\bottomrule
\end{tabular}
\caption{Concept sizes (first line) and F1-scores (second line) of concept learning tools on the SML-benchmarks.}
  \label{tab:generalization}
\end{table*}

We compare the ability of our implementation to produce $\ALCQfeat$ concepts that generalize to unseen examples on the standard SML-benchmarks with the state-of-the art concept learning implementations EvoLearner and TDL\footnote{
We use the implementations of these systems available in OntoLearn \url{https://github.com/dice-group/Ontolearn} development version fd38beec65e85e56e50733d404360b1feb5f5a68, as TDL of version 0.9.2 of OntoLearn reported wrong accuracies} as well as the algorithm underlying Theorem~\ref{thm:optimal}. We do not compare against the standard DL-Learner algorithms CELOE and OCEL, as EvoLearner is known to deliver concepts of higher F1-score~\cite{DBLP:conf/www/HeindorfBDWGDN22}. We configured all systems to learn $\ALCQfeat$ concepts with number restrictions up to a cardinality of $3$. In the case of TDL this means disabling nominals. This places TDL at a disadvantage as it does not support numerical features (only Boolean features). We allowed all tools to run for 5 minutes. TDL and the algorithm underlying Theorem~\ref{thm:optimal} always terminated before; for the anytime algorithms EvoLearner and our tool, we used the best concept produced so far.

Table~\ref{tab:generalization} shows the results (mean and standard deviation over all runs) of 10-fold cross validation measuring the resulting concept size and F1-score on each data set.
We omit results for the Mutagenesis and Nctrer benchmarks, as
they possess a small perfect fitting $\ALCQfeat$ concept that is almost always discovered by the systems. We remark that no implementation performs particularly well on the Carcinogenesis benchmark, where already the concept $\top$ achieves an F1-score of 0.71.

Overall these results demonstrate that bounded fitting performs competitively with state-of-the-art concept learners and is suitable to produce good concepts in practice.  
As expected, the algorithm underlying Theorem~\ref{thm:optimal} produces very large concepts that overfit and do not generalize.

\subsection{Evaluating Fitting for Number Restrictions}

We compare the ability of our implementation to find fitting $\ALCQ$ concepts with that of EvoLearner and TDL. To this end, we created a new benchmark which contains difficult concept learning problems, where qualified number restrictions are essential to separate the positive examples from the negative examples. That is, it contains only instances $I,P,N$ such that there is an $\ALCQ$ concept fitting $P,N$, but no \ALC-concept.

For this, we extract examples from the YAGO 4.5 knowledge base~\cite{DBLP:conf/sigir/SuchanekABCPS24}, focusing on the relations \texttt{children}, \texttt{spouse}, and \texttt{gender}.
We first construct from each $\ALC$-bisimulation equivalence class a single benchmark, by partitioning the equivalence class into $\ALCQ$ bisimulation equivalence classes, taking representatives of half the $\ALCQ$ bisimulation equivalence classes as positive examples and representatives of the other half as negative examples. In a second step, we generate benchmarks with a larger number of examples by combining the positive and negative examples of two benchmarks generated in the first step. 
 We generate $120$ benchmarks with the number of examples ranging from $4$ to $35$. An example of a fitting $\ALCQ$ concept in these benchmarks is the concept $(\le 3\  \texttt{children}.\top)\sqcap (\le 1\ \texttt{children}.\exists\ \texttt{spouse}.\top)$. 
We believe that these benchmarks are of independent interest to evaluate the ability of concept learning implementations to find good $\ALCQ$ on real-world data.

\begin{figure}[!t]
 \centering
    \begin{tikzpicture}
           \begin{axis}[
                   xmin = 0,
                   xmax = 40,
                   xtick = {0,4,...,40},
                   ymin = 0.5,
                   ymax = 1.1,
                   width=8cm, height=5cm,
                   grid = major,
                   grid style={dashed, gray!30},
                   ylabel=\small accuracy / f1-score,
                   xlabel=\small number of examples,
                   ylabel style={yshift=-0.3cm},
                   legend style={outer sep=0pt, at={(0,0), anchor = north east,  font=\tiny,/tikz/every even column/.append style={
            anchor=west 
        }}, nodes={anchor=west},legend columns=2, anchor=north},
                   legend pos=north east,
                    draw=none
                ]                     
               \addplot[only marks, mark =o] table [col sep=comma, y = a_avg_alcsat, x=m] {data/data_avg.csv};
               \addlegendentry{/}
              
               \addplot[only marks, mark =-] table [col sep=comma, y = f_avg_alcsat, x=m] {data/data_avg.csv};   
               \addlegendentry{Our Tool}
              
                \addplot[only marks, mark =o, red] table [col sep=comma, y = a_avg_evo, x=m] {data/data_avg.csv};
                \addlegendentry{/}
               \addplot[only marks, mark =-, red] table [col sep=comma, y = f_avg_evo, x=m] {data/data_avg.csv}; 
               \addlegendentry{EvoLearner}
                \addplot[only marks, mark =o, blue] table [col sep=comma, y = a_avg_tdl, x=m] {data/data_avg.csv};
                 \addlegendentry{/}
               \addplot[only marks, mark =-, blue] table [col sep=comma, y = f_avg_tdl, x=m] {data/data_avg.csv};
               \addlegendentry{TDL}

           \end{axis}
       \end{tikzpicture}
    \caption{Accuracies (circles) and F1-scores (hyphens) for $\ALCQ$.}
    \label{fig:exp-fitting}   
    \vspace{-1em}
\end{figure}
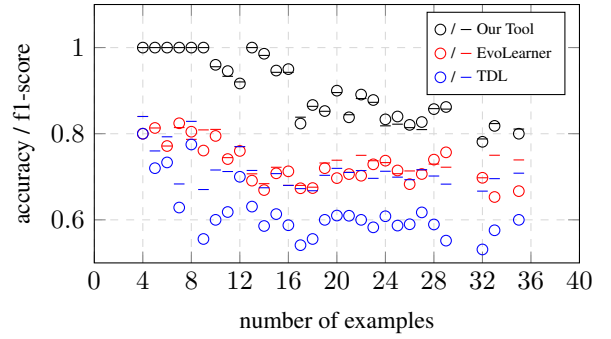

Figure~\ref{fig:exp-fitting} shows the result of comparing mean accuracy and F1-score of the best concept produced within 5 minutes (as in Sec.~6.1) by TDL, EvoLearner (configured to optimize accuracy or F1-score), and our tool on  these benchmarks, ordered by increasing number of examples.
In all benchmarks, our tool is able to find concepts of higher accuracy and F1-score, in some cases even discovering perfect fittings. This indicates that our approach to handle qualified number restrictions has advantages over other tools, for fitting real-world examples.

\subsection{Effect of Features Value Selection}

We evaluate the effect of our approach to fit $\ALCQfeat$ concepts by reduction to the fitting problem for $\ALCQ$. For this, we run our implementation of bounded fitting up to size 8 on SML-benchmarks and vary the number of selected feature values. We use the SML-Benchmarks Mammographic, Mutagenesis, and Suramin, which contain features with domains of size $74$, $529$, and $157$, respectively. Other SML-Benchmarks do not use numeric features to the same extent.

\begin{table}[!t]
\centering
\footnotesize
\begin{tabular}{rrrrrrr}
\toprule
& \multicolumn{2}{c}{Mammographic} & \multicolumn{2}{c}{Suramin} & \multicolumn{2}{c}{Mutagenesis} \\
Values & \textsc{acc} & t & \textsc{acc} & t & \textsc{acc} & t \\
\midrule 
0 & 0.79 & 5.6 & 0.82 & 9.7 & 0.88 & 13 \\
2 & 0.80 & 24 & 0.82 & 11.7 & 0.95 & 8.9 \\
5 & 0.84 & 24.2 & 0.82 & 14.2 & 1.00 & 2 \\
10 & 0.84 & 10.5 & 0.88 & 14.9 & 1.00 & 2.3 \\
20 & 0.84 & 15.8 & 0.88 & 16.9 & 1.00 & 0.3 \\
1000 & 0.84 & 47.9 & 0.94 & 16.9 & 1.00 & 0.6 \\\bottomrule
\end{tabular}
\caption{Accuracy and running time in seconds of bounded fitting up to size 8 with different numbers of values. Average of 3 runs.
}
\label{tab:number-of-intervals}
\vspace{-1em}
\end{table}

The measured accuracy on the training data and running time in seconds for different number of values are shown in Table~\ref{tab:number-of-intervals}, averaged over three runs. The first row for 0 values reports results when only Boolean features are used.
As expected, increasing the number of values allows bounded fitting to find concepts that achieve higher accuracy, but this generally also increases running time, due to additional concept names. 
However, the running time does not always monotonically increase,
as the running time of the SAT solver varies unpredictably with small changes.
These results show that our approach to encode values to obtain a propositional formula is effective, and that on the SML-benchmark already a small number of such values suffices.

\subsection{Effect of Optimizations}

We evaluate the effect of the two presented optimizations.
First, we run our implementation of bounded fitting up to size~8 on SML-benchmarks, and measure the required running time with and without the bisimulation-based preprocessing. We also vary the fragment of $\ALCQIfeat$ and measure the number of individuals before and after preprocessing.

\begin{table}[!t]
\centering
\footnotesize
\begin{tabular}{ccrrrrr}\toprule
   & \multicolumn{2}{c}{\scriptsize Mammographic} & \multicolumn{2}{c}{\scriptsize Hepatitis} & \multicolumn{2}{c}{\scriptsize Lymphography} \\
  & t & \scriptsize $|I|$ & t & \scriptsize $|I|$ & t & \scriptsize $|I|$ \\
 \midrule
 \multirow{2}{*}{\scriptsize \ALC}  &  10  & 975 & 36.9 & 6812  & 18.9 & 148 \\ 
  & 5.8 & 94 & 1.6 & 9 & 19 & 148\\\midrule
 \multirow{2}{*}{\scriptsize \ALCfeat}   &  17 & 975 & 936.3 & 6812 & 19.8& 148\\
  & 9.2 & 364 & 502 & 3720 & 19.5 & 148 \\\midrule
  \multirow{2}{*}{\scriptsize \ALCQfeat}  & 19.4 & 975 & 3668.9 & 6812 & 19.3 & 148 \\
  & 10.5 & 364 & 2816.7 & 5253 & 19.3 & 148 \\\midrule
  \multirow{2}{*}{\scriptsize\ALCQIfeat}  &  60.9 & 975 & 9063.4 & 6812 & 18.8 & 148\\
  &  40.7 & 975 & 9123.6 & 6812 & 18.7 & 148 \\\bottomrule
\end{tabular}
\caption{Running time in seconds of bounded fitting up to size 8 and domain sizes without preprocessing (first line) and with preprocessing (second line). Average of 3 runs.}
\label{table:exp-bisim}
\vspace{-1em}
\end{table}

Table~\ref{table:exp-bisim} shows 
the results on the Mammographic, Hepatitis and Lymphography benchmarks. Again, we report the average of the running time over three runs.
The measured number of individuals shows that on some benchmarks like Lymphography, the preprocessing has no effect, but also does not slow down the implementation.
On other benchmarks, like Mammographic, the number of individuals is reduced which directly results in faster computation.
Naturally, when fitting an $\ALC$ concept more individuals can be identified as when fitting a more expressive concept. This indicates that exploiting bisimilarity is an effective preprocessing step for fitting algorithms, especially when fitting a concept from a less expressive logic.

Second, we determine the efficacy of the implemented parallelization scheme by comparing the running times of our implementation with different number of threads on a selection of the SML-benchmarks.
We measure the needed running time on the Mammographic, Carcinogenesis and Lymphography benchmarks  to run bounded fitting up to size 8 with $t=1, 2, 4$ and $8$ worker threads.

We report full results in the appendix, and only mention the results on the Lymphography benchmark here:
the mean measured running time over three runs with one worker thread is 51.8 seconds, with two worker threads 30.8 seconds, with four worker threads 21.5 seconds and with 8 worker threads 19.6 seconds. Thus, the running times decrease when the number of threads is increased. This effect is strongest when going from 1 to 2 threads.

These results indicate that our parallelization scheme is effective in making use of multiple threads, and enables our implementation to find larger concepts in practical time frames.

\section{Conclusion}

Our theoretical investigation has shown how to
extend bounded fitting to handle inverse roles, qualified number restrictions, and feature comparisons, to preserve theoretical guarantees about its generalization abilities. Our 
practical evaluation confirms our findings and shows that bounded fitting is a promising approach for description logic concept learning, also for expressive description logics. 

We believe that new benchmarks are needed to better evaluate single features of the implemented systems. We made a first step into this direction by providing a new benchmark using which we can test the ability of concept learners to deal with qualified number restrictions. We would like to create similar benchmarks for inverse roles and feature comparisons.

\section*{Ethical Statement}

There are no ethical issues.

\clearpage
\section*{Acknowledgments}
Jean Christoph Jung and Tom Voellmer were supported by DFG grant JU 3197/1-1.

\bibliographystyle{named}
\bibliography{main}

\cleardoublepage

\appendix


\section{Proofs for Section~\ref{sec:fitting}}

We give the remaining details for proving Theorems~\ref{thm:fitting} and Theorem~\ref{thm:optimal}. We recall first the following characterization of fitting for single individuals~\cite{DBLP:journals/sLogica/Rijke00}.
\begin{lemma}
  Let $I$ be a database and $a,b\in\adom(I)$. Then there is an \ALCQ concept fitting $\{a\},\{b\}$ in $I$ iff $I,a\not\sim I,b$.
\end{lemma}

It follows that there is an \ALCQ concept fitting $P,N$ in a database $I$ iff $I,a\not\sim I,b$ for every $a\in P$ and $b\in N$. For the non-trivial ``if''-direction, suppose $I,a\not\sim I,b$ for all $a\in P,b\in N$ and let $C_{ab}$ be the \ALCQ concepts witnessing this, that is, $a\in C_{ab}^{I}$ but $b\notin C_{ab}^{I}$, for all $a\in P,b\in N$. Then
\begin{equation}
\bigsqcup_{a\in P}\bigsqcap_{b\in N}C_{ab}\label{eq:overfit}
\end{equation}
is an \ALCQ concept fitting $P,N$ in $I'$.

Hence, to show Theorem~\ref{thm:fitting} for $\ALCQ$, it suffices to prove the following lemma. It is important to note that we assume throughout this section that concepts are \emph{represented succinctly}, that is, we consider DAG representation which essentially counts the number of non-isomorphic sub-concepts (rather than the number of nodes in the syntax trees). This assumption is necessary to get efficient constructions: otherwise it is known that there families of inputs to the fitting problem for which the smallest fitting concepts is of exponential size~\cite{DBLP:conf/aiml/FigueiraG10}.
\begin{restatable}{lemma}{lemalcqbisi}\label{thm:alcqbisi}
    There is a polynomial time algorithm that decides, given a database~$I$ and $a,b\in\adom(I)$, whether $I,a\sim I,b$. Moreover, in case $I,a\not\sim I,b$, it outputs the succinct representation of a concept $C_{ab}$ with $a\in C_{ab}^I$ and $b\notin C_{ab}^I$.
\end{restatable}

\begin{proof}
    We adapt the partition refinement algorithm that computes the maximal \ALC bisimulation~\cite{DL-Textbook}. The algorithm computes a sequence of equivalence relations $\rho_0,\rho_1,\ldots$ on $\adom(I)$. It starts with
    \[\rho_0 = \{(a,b)\mid A(a)\in I\Leftrightarrow A(b)\in I\text{ for all $A\in \NC$}\}\]
    To compute $\rho_{i+1}$ from $\rho_i$, we rely on the notion of profiles. Let $a\in \adom(I)$, $r\in \NR$, and $X_1,\ldots,X_m$ be an enumeration of all equivalence classes of $\rho_i$. The \emph{$r$-profile of $a$ in $I$ with respect to $\rho_i$} is the tuple $\mn{prof}_I    ^r(a,\rho_i) =(n_1,\ldots,n_m)$ where $n_i$ is the number of $r$-successors of $a$ in class $X_i$, that is, $n_i = |\mn{succ}^r_I(a)\cap X_i|$. We obtain $\rho_{i+1}$ from $\rho_i$ be eliminating all pairs $(a,b)\in\rho_i$ such that $\mn{prof}_I^r(a,\rho_i)\neq \mn{prof}_I^r(b,\rho_i)$, for some $r\in \NR$. Clearly, $\rho_{i+1}$ is an equivalence relation as well.

    By construction, we have $\rho_0\supseteq \rho_1\supseteq \rho_2\supseteq \ldots$. Let $\rho^*$ be where the sequence stabilizes. 

    \medskip\noindent\textbf{Claim 1.} $\rho^*$ is an $\ALCQ$ bisimulation.
    
    \medskip\noindent\textbf{Proof of Claim 1.} By construction, $\rho^*\subseteq \rho_0$ and hence satisfies the Condition~1 of an \ALCQ bisimulation. To see that $\rho^*$ satisfies Condition~2, let $(a,b)\in \rho^*$ and take any $r\in \NR$. Since $(a,b)$ was not eliminated, we have $\mn{prof}_I^r(a,\rho^*)=\mn{prof}_I^r(b,\rho^*)$. Thus, there is a bijection $\pi$ between $\mn{succ}_I^r(a)$ and $\mn{succ}_I^r(b)$ which preserves $\rho^*$. It should be clear that for any $D\subseteq\mn{succ}_I^r(a)$, the image $\pi(D)\subseteq \mn{succ}_I^r(b)$ witnesses that Condition~2 is satisfied. Condition~3 is analogous. This finishes the proof of Claim~1.
    
 \medskip\noindent\textbf{Claim 2.} $S\subseteq\rho^*$ for any $\ALCQ$ bisimulation $S$ on $I$.
    
    \medskip\noindent\textbf{Proof of Claim 2.} The proof is by induction on the number of stages of the elimination. In the inductive base, we clearly have $S\subseteq \rho_0$. To see that $S\subseteq \rho_{i+1}$ whenever $S\subseteq \rho_i$, take any $(a,b)\in S$ and any $r\in\NR$. We aim to show that $\mn{prof}_I^r(a)=\mn{prof}_I^r(b)$. So consider any class $X_i$ of $\rho_i$. By induction, $X_i$ is a partition of $S$-equivalence classes $Y_1,\ldots,Y_k$. Obviously, 
    \[\mn{succ}_I^r(a)\cap X_i = \bigcup_{j=1}^k(\mn{succ}_I^r(a)\cap Y_j).\]
    Since $S$ satisfies Conditions~2 and~3, we find bijections between $\mn{succ}_I^r(a)\cap Y_j$ and $\mn{succ}_I^r(b)\cap Y_j$, for every $j\in\{1,\ldots,k\}$. These bijections can be assembled to a bijection between 
    $\mn{succ}_I^r(a)\cap X_i$ and $\mn{succ}_I^r(b)\cap X_i$, showing that these sets have the same cardinality. Overall, this shows that $\mn{prof}_I^r(a,\rho_i)=\mn{prof}_I^r(b,\rho_i)$, and hence $(a,b)$ is not eliminated. This finishes the proof of Claim~2.

    \bigskip It is an immediate consequence of Claims~1 and~2 that $\rho^*$ is the maximal \ALCQ bisimulation in $I$. Moreover, since the size of $\rho_0$ is bounded by $|\adom(I)|^2$, in each stage at least one pair is eliminated, and the elimination condition (different profiles for some $r$) can be decided in polynomial time, the algorithm runs in polynomial time. 

    It remains to argue that, for every eliminated pair $(a,b)$, we can compute (the representation of) a concept $C_{ab}$ with $a\in C_{ab}^I$, but $b\notin C_{ab}^I$. We proceed by induction on the number of rounds of the elimination algorithm.

    In the inductive base, consider any pair $(a,b)\notin \rho_0$. Clearly, there is a concept name $A\in \NC$ such that either $A(a)\in I$, but $A(b)\notin I$ or $A(b)\in I$, but $A(a)\notin I$. Then $A$ or $\neg A$, respectively, is a fitting concept $C_{ab}$. 

    In the inductive step, suppose that $(a,b)\in\rho_i\setminus\rho_{i+1}$, that is, $(a,b)$ got eliminated in round $i+1$. By the elimination condition, we have $\mn{prof}_I^r(a,\rho_i)\neq \mn{prof}_I^r(b,\rho_i)$, for some $r\in \NR$. Assume first that there is an equivalence class $X$ of $\rho_i$ such that
    \begin{equation}
        |\mn{succ}_I^r(a)\cap X|<|\mn{succ}_I^r(b)\cap X|.\label{eq:case-construction}
    \end{equation}
    %
    By induction, there is a separating concept $C_{de}$ for every $d\in X, e\in \overline X:=\adom(I)\setminus X$. Consider the concept $C_{ab}=(\leq k\ r.D_X)$ where $D_X$ is defined by taking:
    \[D_X=\bigsqcup_{d\in X}\bigsqcap_{e\in \overline X }C_{de}.\]
    It should be clear that $a\in C_{ab}^I$, but $b\notin C_{ab}^I$. Moreover, the number of subconcepts of each $C_{ab}$ is polynomial in the size of the input. 

    In the other case, when~\eqref{eq:case-construction} is true for some $X$ with the roles of $a,b$ switched, we construct in the same way a concept $C_{ba}$ with $a\notin C_{ba}^I$ and $b\in C_{ba}^I$. Then, $C_{ab}=\neg C_{ba}$ is as required. 
\end{proof}

\lemreduction*

\begin{proof}
    The proof of correctness is rather straightforward. We give here only the idea to see that it is also constructive. 
    
    \smallskip
    We start with the first reduction.
    
    $(\Rightarrow)$ Let $C$ be an \ALCQIfeat concept $C$ fitting $P,N$ in $I$. We let $\overline C$ denote the \ALCQfeat concept obtained from $C$ by replacing any inverse role $r^-$ by $\overline r$. It is routine to verify that $C^I=\overline C^{J}$ for every \ALCQIfeat concept $C$, hence $\overline C$ fits $P,N$ in $J$.
    
     $(\Leftarrow)$ Let $C$ be an \ALCQfeat concept $C$ fitting $P,N$ in $J$. We let $\overline C$ denote the \ALCQIfeat concept obtained from $C$ by replacing any $\overline r$ by the inverse role $r^-$. Similar as above, we have $C^J=\overline C^{I}$ for every \ALCQfeat concept $C$, hence $\overline C$ fits $P,N$ in $I$.
    
     \smallskip
     Consider now the second reduction.

     $(\Rightarrow)$ Let $C$ be an \ALCQfeat concept $C$ fitting $P,N$ in $I$. Assume without loss of generality that feature comparisons in $C$ are all of shape $(f\geq v)$. We let $\overline C$ denote the \ALCQ concept obtained from $C$ by replacing any feature comparison $(f\geq v)$ by the disjunction of all concepts $A_{f\geq w}$ with $w\geq v$. It is routine to verify that $C^I=\overline C^{J}$ for every \ALCQfeat concept $C$, hence $\overline C$ fits $P,N$ in $J$.

     $(\Leftarrow)$ Let $C$ be an \ALCQ concept $C$ fitting $P,N$ in $J$. We let $\overline C$ denote the \ALCQfeat concept obtained from $C$ by replacing any concept name $A_{f\geq v}$ by the feature comparison $(f\geq v)$. It is routine to verify that $C^J=\overline C^{I}$ for every \ALCQ concept $C$, hence $\overline C$ fits $P,N$ in $I$.
\end{proof}

We come now to the proof of Theorem~\ref{thm:optimal}.

\thmoptimal*

\begin{algorithm}[t]
\caption{Approximate Fitting in \ALCQ.}\label{alg:approximatefitting}
   \KwIn{Database $I$, examples $P,N$, integer $k \geq0$}
   $S \coloneqq$ maximal \ALCQ bisimulation on $I$

   $P' \coloneqq \emptyset$, \quad $N' \coloneqq \emptyset$

   \ForEach{equivalence class $X$ of $S$}{

        \If{$|P\cap X|\geq |N\cap X|$}{$P' := P'\cup (P\cap X)$}
        
        \Else{$N' \coloneqq N'\cup (N\cap X)$}
        
    }
   \If{$k\leq |P'|+|N'|$}{\Return yes}
   
   \Return no   
\end{algorithm}
\begin{proof}
    We prove the theorem for \ALCQ. It follows for \ALCQIfeat via the reduction provided in (the proof of) Lemma~\ref{lem:reduction}. The algorithm deciding approximate fitting in \ALCQ is given in Algorithm~\ref{alg:approximatefitting}. Intuitively, it decides for every equivalence class $X$ of the maximal \ALCQ bisimulation $S$ in $I$ whether to keep the positive or the negative examples in that class $X$. Clearly, at all times there is an \ALCQ concept fitting $P',N'$ in $I$. Hence, if $|P'|+|N'|$ is at least $k$ in the end, the input was a yes-instance. The (succinct representation of the) \ALCQ concept fitting $P',N'$ can be computed as in Equation~\eqref{eq:overfit}. The algorithm runs in polynomial time since the maximal bisimulation can be computed in polynomial time. 
\end{proof}

\section{Proofs for Section~\ref{sec:boundedfitting}}

\thmsizefittingnpcomplete*

\begin{proof}
  We give a reduction of \emph{hitting set} which is the problem to decide, given a  collection of sets $S=\{S_1,\dots,S_m\}$ and a size bound $k\in\mathbb N$, whether there exists a set $H$ such that $|H|\le k$ and $H\cap S_j\ne\emptyset$, for each $j\in\{1,\dots,m\}$.  The set $H$ is called a hitting set for $S$. The reduction closely follows the one in the proof of Proposition~1 from~\cite{ISWC25} for $\ALC$, but its proof of correctness is more difficult due to inverse roles and qualified number restrictions.

Let $(S,k)$ be an instance of hitting set with $S=\{S_1,\dots,S_m\}$ a collection of sets $S_i\subseteq\mathbb{N}$ and $k\in\mathbb N$. Assume without loss of generality that $\bigcup_{j=1}^m S_j = \{1,\dots,n\}$ for some $n\in\mathbb N$. We construct a database $I$ in which contains only facts involving two role names $r,s$ and a concept name $A$, and such that there are two elements $a,b\in\adom(I)$, such that for $P=\{a\}$, $N=\{b\}$, and $k'=k+n+2$:
\begin{align}
    & \text{$S$ has a hitting set of size $k$}\quad\Leftrightarrow \notag\\
    &
    \quad \quad \text{$P,N$ admit an $\ALCQIfeat$ fitting of size $k'$ in $I$.}\label{eq:correctness}
\end{align}
We first provide an informal description of $I$. Its main components are certain $r$-paths of length $n$. The nodes of the paths are \emph{groups} of $k'+1$ individuals, which we denote with bold face letters $\abf,\bbf$, possibly with indices. An $r$-edge between two groups $\abf_1$ and $\abf_2$ indicates $r$-edges between any $a_1$ and $a_2$ with $a_1\in \abf_1$ and $a_2\in \abf_2$. All paths end in a group satisfying $A$. To encode a subset $S'\subseteq \{1,\ldots,n\}$, such a path is extended with additional ``detours'' along role $s$. In more detail, suppose such an $r$-path consists of groups $\bbf_0,\ldots,\bbf_n$. We encode a subset $S'\subseteq \{1,\ldots,n\}$ by including, for all $i\in\{1,\ldots,n\}\setminus S'$, an $s$-path of length $2$ from $\bbf_i$ to $\bbf_{i+1}$.
The database $I$ consists of 
\begin{itemize}

\item one path for each $S_j\in S$ and an additional root $\bbf$ which connects via $r$ to the initial nodes of all the paths; 

\item one additional path that encodes the empty set, and has root $\abf$ which connects via $r$ to the initial nodes of all paths.

\end{itemize}
Figure~\ref{fig:np} illustrates the constructed database for the instance $(S,2)$ where $S=\{ \{1,3\}, \{2,4\}\}$. The upper path from $\abf_0$ to $\abf_4$ is the mentioned path encoding the empty set, the path from $\bbf_{1,0}$ to $\bbf_{1,4}$ encodes $\{1,3\}$, and the path from $\bbf_{2,0}$ to $\bbf_{2,4}$ encodes $\{2,4\}$. For this input, $H=\{1,2\}$ is a hitting set of size $2$.

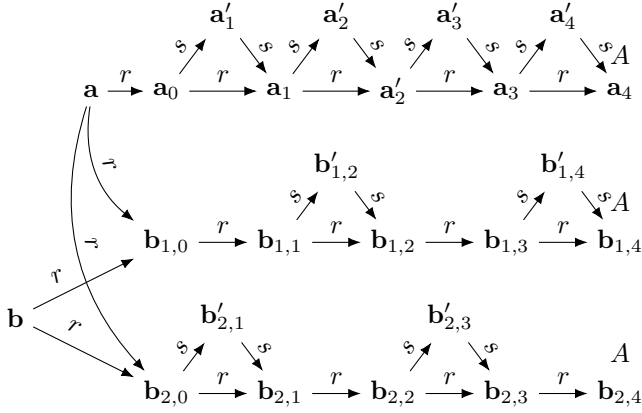
\begin{figure}[t]
            \begin{center}

            \begin{tikzpicture}
            \node[label = ${ }$] (a) at (-1,-2.0) {$\abf$};
            \node[label = ${ }$] (a0) at (0.0,-2.0) {$\abf_{0}$};
            \node[label = ${ }$] (a1) at (1.5,-2.0) {$\abf_{1}$};
            \node[label = ${ }$] (ap1) at (0.75,-1.0) {$\abf_{1}'$};
            \node[label = ${ }$] (a2) at (3.0,-2.0) {$\abf_{2}'$};
            \node[label = ${ }$] (ap2) at (2.25,-1.0) {$\abf_{2}'$};
            \node[label = ${ }$] (a3) at (4.5,-2.0) {$\abf_{3}$};
            \node[label = ${ }$] (ap3) at (3.75,-1.0) {$\abf_{3}'$};
            \node[label = ${A }$] (a4) at (6.0,-2.0) {$\abf_{4}$};
            \node[label = ${ }$] (ap4) at (5.25,-1.0) {$\abf_{4}'$};
            

            \draw[-Latex] (a0) edge node[above, sloped] {$ r$} (a1);
            \draw[-Latex] (a0) edge node[above, sloped] {$ s$} (ap1);
            \draw[-Latex] (a1) edge node[above, sloped] {$ r$} (a2);
            \draw[-Latex] (a1) edge node[above, sloped] {$ s$} (ap2);
            \draw[-Latex] (ap1) edge node[above, sloped] {$ s$} (a1); 
            \draw[-Latex] (a2) edge node[above, sloped] {$ r$} (a3);
            \draw[-Latex] (a2) edge node[above, sloped] {$ s$} (ap3);
            \draw[-Latex] (ap2) edge node[above, sloped] {$ s$} (a2);
            \draw[-Latex] (a3) edge node[above, sloped] {$ r$} (a4);
            \draw[-Latex] (a3) edge node[above, sloped] {$ s$} (ap4);
            \draw[-Latex] (ap3) edge node[above, sloped] {$ s$} (a3);
            \draw[-Latex] (ap4) edge node[above, sloped] {$ s$} (a4);

                    \node[label = ${ }$] (b) at (-2,-5.0) {$\bbf$};    
                    \node[label = ${ }$] (b10) at (0,-4.0) {$\bbf_{1,0}$};                    
                    \node[label = ${ }$] (b20) at (0,-6.0) {$\bbf_{2,0}$};
  
                    \node[label = ${ }$] (bp21) at (0.75,-5.0) {$\bbf_{2,1}'$};
                    \node[label = ${ }$] (b11) at (1.5,-4.0) {$\bbf_{1,1}$};
                    \node[label = ${ }$] (b21) at (1.5,-6.0) {$\bbf_{2,1}$};
                   
                    \node[label = ${ }$] (bp12) at (2.25,-3.0) {$\bbf_{1,2}'$};
                    \node[label = ${ }$] (b12) at (3.0,-4.0) {$\bbf_{1,2}$};
                    \node[label = ${ }$] (b22) at (3.0,-6.0) {$\bbf_{2,2}$};
                    
                    \node[label = ${ }$] (bp23) at (3.75,-5.0) {$\bbf_{2,3}'$};
                    \node[label = ${ }$] (b13) at (4.5,-4.0) {$\bbf_{1,3}$};
                    \node[label = ${ }$] (b23) at (4.5,-6.0) {$\bbf_{2,3}$};
                            
                    \node[label = ${ }$] (bp14) at (5.25,-3.0) {$\bbf_{1,4}'$};
                    \node[label = ${A }$] (b14) at (6.0,-4.0) {$\bbf_{1,4}$};
                    \node[label = ${A }$] (b24) at (6.0,-6.0) {$\bbf_{2,4}$};
                                                
                    \draw[-Latex] (b) edge node[sloped, above left = 2pt] {$ r$} (b10); 
                    \draw[-Latex] (b) edge node[sloped, above left = 2pt] {$ r$} (b20); 
                    
                    \draw[-Latex] (b10) edge node[above] {$ r$} (b11); 
                  
                    \draw[-Latex] (bp21) edge node[above, sloped] {$ s$} (b21); 
                    \draw[-Latex] (b11) edge node[above, sloped] {$ r$} (b12); 
                    \draw[-Latex] (b11) edge node[above, sloped] {$ s$} (bp12); 
                    \draw[-Latex] (b20) edge node[above, sloped] {$ r$} (b21); 
                    \draw[-Latex] (b21) edge node[above, sloped] {$ r$} (b22); 
                                                        
                    \draw[-Latex] (bp12) edge node[above, sloped] {$ s$} (b12); 
                    \draw[-Latex] (b12) edge node[above, sloped] {$ r$} (b13); 
           
                    \draw[-Latex] (b22) edge node[above, sloped] {$ r$} (b23); 
                    \draw[-Latex] (b22) edge node[above, sloped] {$ s$} (bp23);              
                    \draw[-Latex] (bp23) edge node[above, sloped] {$ s$} (b23); 
                    \draw[-Latex] (b20) edge node[above, sloped] {$ s$} (bp21); 
                    \draw[-Latex] (b13) edge node[above, sloped] {$ r$} (b14); 
                    \draw[-Latex] (b13) edge node[above, sloped] {$ s$} (bp14); 
                    \draw[-Latex] (b23) edge node[above, sloped] {$ r$} (b24); 


                   
                    \draw[-Latex] (bp14) edge node[above, sloped] {$ s$} (b14); 
                          
                    \draw[-Latex] (a) edge node[above, sloped] {$r$} (a0);
                    \draw[-Latex] (a) edge[bend right] node[above, sloped] {$r$} (b10);
                    \draw[-Latex] (a) edge[bend right] node[above, sloped] {$r$} (b20);



                    
                    \end{tikzpicture}
                    \end{center}
            \caption{Example database $I$ used in the reduction.}
        \label{fig:np}
        \end{figure}

        We now make the definition of $I$ for a given $(S,k)$ more precise. In the definition below, a fact $A(\abf)$ means that $A(a)$ for all $a\in \abf$, and a fact $r(\abf,\bbf)$ is an abbreviation for $r(a,b)$, for all $a\in\abf,b\in\bbf$. The domain of $I$ is given as follows: 
\begin{align*}
  \adom(I) ={}& \{\abf\}\cup \{\abf_{i},\abf_i'\mid 0\le i\le n\} \cup{}\\
  & \{\bbf_{j,i},\bbf_{j,i}'\mid 1\le j \le m, 0\le i \le n\}.
  %
\end{align*}
The $A$-facts mark the end of the paths: 
\[A(\abf_n)\text{ and }A(\bbf_{j,n}), \text{ for all $j\in m$.}\]
The $r$-facts can be grouped as follows:
\begin{itemize}
    \item $r(\abf,\abf_0),r(\abf,\bbf_{j,0})$, and $r(\bbf,\bbf_{j,0})$, connect the root elements $\abf,\bbf$ to the starts of the paths; 
    \item $r(\abf_{i-1},\abf_i)$ for all $i\in\{1,\ldots,n\}$ establish the path along the $\abf$-elements; and
    \item $r(\bbf_{j,i-1},\bbf_{j,i})$ for all $i\in\{1,\ldots,n\}$ and $j\in\{1,\ldots,m\}$ establish the paths for representing the sets in $S$. 
\end{itemize}
The $s$-facts are given as follows: 
\begin{itemize}
    \item $s(\abf_{i-1},\abf_i'),s(\abf_i',\abf_i)$, for all $i\in\{1,\ldots,n\}$, describes the $s$-detours along the $\abf$-path;

    \item $s(\bbf_{j,i-1},\bbf_{j,i}'),s(\bbf_{j,i}',\bbf_{j,i})$ for all $j\in\{1,\ldots,m\}$ and all $i\in \{1,\ldots,n\}$ with $i\notin S_j$, describes the $s$-detours along the $\bbf$-paths. 
    
\end{itemize}

We pick some $a\in \abf$ and some $b\in\bbf$ as single positive and negative examples, respectively, that is, $P=\{a\},N=\{b\}$. In our proof below, we will use the following properties of $I$:
\begin{enumerate}[label=(P\arabic*)]

%

    \item\label{it:bisimilar} The individuals in any group $\abf_*/\bbf_*$ (are $\ALCQI$-bisimilar and thus) satisfy the same \ALCQI concepts. Hence, the following equivalences holds for every role or inverse role $R$ and any \ALCQI concept $E$:
    \begin{align*}
    (\exists R.\top\sqcap \forall R.E)^I & = (\exists R.E)^I. \\
    (\neg \exists R.\top\sqcap \forall R.E)^I & = \top^I.
    \end{align*}
    
    \item\label{it:nonumbers} Since every group contains $k'+1$ individuals, the following equivalences hold for any number $1\leq M\leq k'$, any role or inverse role $R$, and any \ALCQI concept $E$: 
    \begin{align*}
        (\geq M\ R.E)^I & = (\geq 1\ R.E)^I\\
        (\leq M\ R.E)^I & = (\leq 0\ R.E)^I
    \end{align*}
    Recall that the concepts on the right-hand side are equivalent to $\exists R.E$ and $\forall R.\neg E$, respectively. 
    
\end{enumerate}

We call a concept of shape $\exists w_1.\ldots\exists w_m.A$ for $w_1,\ldots,w_m\in\{r,s\}$ a \emph{path concept}. We verify Equivalence~\eqref{eq:correctness} by showing that the following are equivalent: 

        \begin{enumerate}[label=(\roman*)]
        
            \item $S$ has a hitting set of size at most $k$; 
        
            \item there is a path concept of shape $\exists r.D$ and of size at most $k'$ fitting $P,N$ in $I$; 
            
            \item there is an $\ALCQIfeat$ concept of size at most $k'$ fitting $P,N$ in $I$;

            \item there is an $\ALCI$ concept of size at most $k'$ fitting $P, N$ in $I$; 
            
            \item there is an $\ALCI$ concept of size at most $k'$ and of shape $\exists r.D$ fitting $P, N$ in $I$.
             
        \end{enumerate}

        We show (i)$\Rightarrow$(ii), then the chain (ii)$\Rightarrow$(iii)$\Rightarrow$(iv)$\Rightarrow$(v)$\Rightarrow$(ii), the first of which is trivially true, and finally (ii)$\Rightarrow$(i).

        \smallskip\noindent\textit{Claim 1.} Implication (i)$\Rightarrow$(ii) holds. 

        \smallskip\noindent\textit{Proof of Claim 1.} Let $H$ be a hitting set for $S=\{S_1,\dots,S_m\}$ with $|H| = k$. We inductively define path concepts $C_i$, for $i=0,\ldots,n$, by setting $C_0=A$ and 
\begin{equation}\label{eq:paths}
       C_i = \begin{cases}
            \exists r.C_{i-1} & \text{if }n-i+1\notin H\\
            \exists s.\exists s.C_{i-1} & \text{otherwise},
        \end{cases}
\end{equation}
        for $1\leq i\le n$. Thus, $\lVert C_n \rVert = n + |H| +1$ and therefore $\lVert\exists r.C_n\rVert = n + k +2 = k'$. 
        We claim that $D=\exists r.C_n$ fits $P,N$ in $I$. In the example in Figure~\ref{fig:np}, the concept constructed for the hitting set $H=\{1,2\}$ is \[D=\exists r.\exists s.\exists s.\exists s.\exists s.\exists r.\exists r.A,\] and it is easily verified that it fits. 
        
        In general, any path concept $E$ that can be obtained this way satisfies $a\in E^I$. Moreover, since $D$ is constructed from a hitting set $H$, the construction forces the path described by $E$ to ``leave'' the paths encoding the sets $S_j$ in $I$. This means that $b$ cannot be in $D^I$. This finishes the proof of Claim~1.

        \smallskip\noindent\textit{Claim 2.} Implication (iii)$\Rightarrow $(iv) holds. 
        
        \smallskip\noindent\textit{Proof of Claim 2.} Let $C$ be an \ALCQIfeat concept of size at most $k'$ that fits $P,N$ in $I$. Since $I$ does not use feature names, it is without loss of generality to assume that $C$ is actually an \ALCQI concept. Then the claim is an immediate consequence of Property~\ref{it:nonumbers} and the observation that there cannot occur any numbers greater than $k'$ in $C$, as it is only of size $k'$. Here we rely on the fact that we consider unary coding. This finishes the proof of Claim~2.
        
        \smallskip\noindent\textit{Claim 3.} Implication (iv)$\Rightarrow $(v) holds. 
        
        \smallskip\noindent\textit{Proof of Claim 3.} Let $C$ be a minimal \ALCI concept that fits $P,N$ in $I$. We assume without loss of generality that $C$ is in negation normal form. We first observe that $C$ is neither a conjunction nor a disjunction: 
\begin{itemize}

    \item If $C$ is a conjunction $C_1\sqcap C_2$, then one of $C_1,C_2$ fits $P,N$ as well since there is a single negative example. This is in contradiction to minimality of $C$.
    
    \item If $C$ is a disjunction $C_1\sqcup C_2$, then one of $C_1,C_2$ fits $P,N$ as well since there is a single positive example. This is in contradiction to minimality of $C$.
    
\end{itemize}
Hence, $C$ is either a concept name $A$, a negated concept name $\neg A$, or of shape $(\exists r.D)$ or $(\forall r.D)$; existential or value restrictions involving $s$ or inverse roles cannot fit. The first two cases are not possible since neither $A$ nor $\neg A$ fit. It remains to note that $\forall r.D$ cannot fit $P,N$ in $I$, since the construction of $I$ implies that $a\in (\forall r.D)^I$ implies $b\in(\forall r.D)^I$.

\smallskip\noindent\textit{Claim 4.} Implication (v)$\Rightarrow $(ii) holds. 

\smallskip\noindent\textit{Proof of Claim 4.} We need some auxiliary notation. For a finite sequence $w=r_1\cdots r_m\in\{r,s\}^*$, we abbreviate with $\exists w.C$ the concept $\exists r_1.\ldots.\exists r_m.C$. Moreover, for $a_0\in \adom(I)$, we let $R_w^I(a_0)$ denote the set of all $b\in\adom(I)$ such that $(a_0,b)\in r_1^{I}\circ\ldots\circ r_m^{I}$, where $\circ$ denotes the composition of binary relations.

Let now be $C=\exists r.D$ be a minimal \ALCI concept fitting $P,N$ in $I$. We inductively construct concepts of shape $C_1=\exists w_1.D_1$, $C_2=\exists w_2.D_2$, \ldots such that each $C_i$ fits $P,N$ in $I$, is of size at most $k'$, and 
\begin{enumerate}

    \item $R_{w_i}^I(a)$ contains a single group $\abf_j$ or $\abf_j'$ and each individual in this group satisfies $D_i$;

    \item no element from $R^I_{w_i}(b)$ satisfies $D_i$.
    
\end{enumerate}
In the inductive base, we set $w_1=r$ and $D_1=D$.
Clearly, Properties~1 and~2 are satisfied by this choice.

\smallskip For the inductive step, take any $C_i=\exists w_i.D_i$ and assume that $D_i$ is minimal with all the properties. We observe first that it is without loss of generality that $D_i$ is not a disjunction. Indeed, if $D_i=E_1\sqcup E_2$, then one of $\exists w_i.E_1,\exists w_i.E_2$ fits $P,N$ as well, still satisfies Properties~1 and~2, and is smaller than $D_i$. This is in contradiction with the minimality of $D_i$.

We investigate now closer the shape of $D_i$, which is a conjunction of $M\geq 1$ concepts $E_1,\ldots,E_M$:
\begin{itemize}

    \item If some $E_i=\neg A$, then $R^I_{w_i}(a)$ does not contain $\abf_n$. This means that also $\bbf_{j,n}\notin R_{w_i}^I(b)$ for any $j$. Then, we can replace $E_i$ with $\top$. 
 
    \item If $M=1$ and $E_1=\top$, we are done.
    
    
    \item There cannot be a conjunct of shape $\exists r^-.E$ (resp., $(\exists s^-.E)$) as such conjunct implies that the last element of the sequence is $r$ (resp., $s$). But then we can remove $\exists r^-.E$ (resp., $\exists s^-.E$) and make $E$ a conjunct of the previous $D_{i-1}$, resulting in a smaller concept.
    

    \item There cannot be a conjunct of shape $\forall r^-.E$ (resp., $(\forall s^-.E)$) as such conjunct implies that the last element of the sequence is $r$ (resp., $s$). But then we can remove $\forall r^-.E$ (resp., $\forall s^-.E$) and make $E$ a conjunct of the previous $D_{i-1}$, resulting in a smaller concept.

    \item If $E_j=A$ for some $j$, then $R_{w_i}^I(a)$ contains $\abf_n$. 
    But $\abf_n$ do not satisfy any concept of shape $\exists r.E$ or $\exists s.E$, hence those cannot be conjuncts. Further, $\abf_n$ trivially satisfies concepts $\forall r.E$ and $\forall s.E$, so they are not there (due to minimality). Hence, $A$ is the only conjunct and we are done. 

    \item By Property~(P1), we can replace $\forall r.E/\forall s.E$ by  $\exists r.E/\exists s.E$ or $\top$, depending whether the group from Property~1 has an $r/s$-successor or not.
   
    \item There cannot be two different conjuncts $\exists r.E,\exists r.E'$ since they can be replaced by the shorter $\exists r.(E\sqcap E')$. The same for $s$.
   
    \item If there is a single conjunct $E_1=\exists t.E_1'$ for $t\in \{r,s\}$, we can just extend $w_i$ with $t$, that is, $w_{i+1}=w_it$ and set $D_{i+1}=D_i'$.

    \item It remains to consider the case that $D_i$ is of shape $\exists r.F_1\sqcap \exists s.F_2$. We distinguish cases on the shape of the group from Property~1.
    
    \begin{itemize}
    
        \item If $\abf_\ell\in R^I_{w_i}(a)$, for some $\ell$, then we can replace $D_i$ with $\exists s.\exists s.F_1$. This is clearly satisfied in $\abf_\ell$. Moreover, in this case, groups in $R_{w_i}^I(b)$ are of shape $\bbf_{j,\ell}$, and none of them satisfies $\exists s.\exists s.F_1$, by the assumption that they do not satisfy $\exists r.F_1\sqcap \exists s.F_2$. Thus Properties~1 and~2 are preserved. 
    
        \item If $\abf_\ell'\in R^I_{w_i}(a)$, for some $\ell$, then we get a contradiction, since $\abf_\ell'$ does not have an $r$-successor and hence does not satisfy $\exists r.F_1$.
        
    \end{itemize}
    
\end{itemize}
This finishes the proof of Claim~4.

\smallskip\noindent\textit{Claim 5.} Implication (ii)$\Rightarrow $(i) holds. 
        
\smallskip\noindent\textit{Proof of Claim 5.} We use again the notation from the previous claim. Let $C=\exists r.\exists w.D$ be a path concept fitting $P,N$ in $I$. We start with making two observations on the shape of $C$: 
\begin{itemize}

    \item $C$ is of shape $\exists r.\exists w.A$ for some $w$: the only other option $D=\top$ is impossible since $b\in(\exists r.\exists w.\top)^I$. 

    \item $w\in (r+ss)^+$ since this is the only way to "reach" $\abf_n$ from $\abf$. Moreover, due to the size bound $\lVert C \rVert \leq k'=n+k+2$, the pattern $ss$ occurs at most $k$ times in $w$. 
    
\end{itemize}
Due to the second item above we can write $w=w_1\ldots w_n$ with $w_i\in\{r,ss\}$ for all $i$ and such that $w_i=ss$ at most $k$ times. We read off a hitting set $H$ of size at most $k$ from $C$ by taking
\[H = \{i\in\{1,\ldots,n\}\mid w_i=ss\}\]
We have to show $H\cap S_i\neq \emptyset$, for every $j$. Suppose, with the goal of deriving a contradiction, that $H\cap S_j=\emptyset$. Then, by the construction of $I$, we have $\bbf_{j,0}\in (\exists w.A)^I$. Hence, $b\in C^I$ and $b$ is not a negative example, contradiction. This finishes the proof of the last claim and hence of the theorem.
\end{proof}

\thmalcqoccam*
\begin{proof}
    Let $g\in\Omega(k)\cap O(2^{p(k)})$ and let $\mathbf A_g$ be Algorithm~\ref{alg:boundedfittingalcq} where $g$ is used to increase values in number restrictions. We seek to bound the VC-dimension of the hypothesis space $\mathcal H^{A_g}(\Sigma,s,m)$ by $p(s,|\Sigma|)m^{\alpha}$ and start with some observations about $\mathcal H^{A_g}(\Sigma,s,m)$. Let $\Sigma$ be any signature, $s,m\in\mathbb N$ and $C_T$ with $\lVert C_T\rVert\le s$ be any target concept. When $\mathbf A_g$ is started on $m$ examples it proceeds by increasing the concept size $k$ incrementally, in each step searching for a fitting $\ALCQ^k_{g(k)}$ concept, that is, an $\ALCQ$ concept whose syntax tree has $k$ nodes and the value of a number restriction appearing in $C$ is bounded by $g(k)$. Because of $g\in\Omega(k)$ Algorithm $A_g $ is complete. Since we assume unary coding of values in number restrictions the maximal value of a number restriction that appears in $C_T$ is bounded by $s$. Thus if $A_g$ is started on examples $P$ and $N$ so that $C_T$ fits $P,N$ and $k$ reaches a value that satisfies $k\ge s$ and $g(k) \ge s$ it returns a fitting concept $D$ with $\lVert D\rVert \le s$, possibly $D=C_T$. Indeed, Algorithm $A_g$ necessarily terminates after the minimal value for $k$ with $g(k)\ge s$ and $k\ge s$ is reached. This allows us to bound the effective hypothesis space $\mathcal H^{\mathbf A_g}(\Sigma,s,m)$ by
    \[|\mathcal H^{\mathbf A_g}(\Sigma,s,m)|\le (|\Sigma|+g(k^\ast)+c)^{k^\ast}\]
    where $k^\ast$ is some value with $g(k^\ast)\ge s$ and $k^\ast\ge s$ to be determined later. 

    Let $d$ be the VC-dimension of $|\mathcal H^{\mathbf A_g}(\Sigma,s,m)|$. Since $2^d$ concepts are required to shatter a set of examples of cardinality $d$ we have
    \[2^d\le|\mathcal H^{\mathbf A_g}(\Sigma,s,m)|\le (|\Sigma|+g(k^\ast+c)^{k^\ast}\] and thus
    \begin{align*}
        d&\le \log ((|\Sigma|+g(k^\ast)+c)^{k^\ast})\\
        &=k^\ast\log (|\Sigma|+g(k^\ast)+c)\\
        &\le k^\ast\log (|\Sigma|g(k^\ast)c)\\
        &= k^\ast(\log (|\Sigma|c) + \log (g(k^\ast))))\\
        &\le k^\ast c'|\Sigma| + k^\ast \log(g(k^\ast))\tag{for an appropriate constant c'}
    \end{align*}
Because of $g\in\Omega(k)\cap O(2^{p(k)})$ there are constants $c_1\in (0,1],c_2\ge 1$ and $k_{1}',k_{2}'>0$ so that $g(k) \ge c_1k$ for all $k\ge k_1'$, $g(k)\ge c_1k$ and for all $k\ge k_2',$ $g(k)\le c_22^{p(k)}$. W.l.o.g. we show the polynomial bound required for an Occam algorithm only for $s\ge\max\{k_1',k_2'\}$ as we can account for any smaller $s$ by an appropriately chosen constant in the polynomial. We set $k^\ast = \frac s{c_1}$, ensuring $g(k^\ast)\ge s$ and $k^\ast\ge s$. Because of $g\in O(2^{p(k)})$ we have 
\begin{align*}
    & k^\ast c'|\Sigma| + k^\ast \log(g(k^\ast))\\
    &\le \frac s{c_1}c'|\Sigma| + \frac s{c_1} \log(c_2 2^{p(\frac s{c_1})})\\
    &\le \frac s{c_1}c'|\Sigma| + \frac s{c_1} \log(c_2) + \frac s{c_1} {p(\frac s{c_1})}\log(2).
\end{align*}
    
\end{proof}

Let $\mathbf A_f$ be the algorithm obtained by first applying the reduction of fitting in \ALCfeat to fitting in \ALC from Lemma~\ref{lem:reduction} and then applying bounded fitting (for \ALC) to the result.

\thmreductionnotoccam*
Intuitively, the hypothesis space $\mathcal{H}^{\mathbf A_f}(\Sigma,s,m)$ contains all concepts of the form $(f=j)$ for $j\in\dom(f)$. It is then possible to construct a set of examples, so that any subset is identified uniquely by a a concept of the form $\exists r.(f=j)$, thus the set is shattered by these concepts. Note that to show this for a set of examples of cardinality exponential in $s$ we also require an exponential number of $r$-successors, making use of unbounded outdegree in databases.
To prove Theorem~\ref{thm:reduction-not-occam} we rely on Theorem~3.2.1 from \cite{DBLP:journals/jacm/BlumerEHW89}, that is, every Occam algorithm defines a sample-efficient PAC learning algorithm, and thus prove the stronger statement, that $\mathbf A_f$ is not a sample-efficient PAC-learning algorithm. Note, that we use a different but equivalent notion of examples here that is more natural when studying PAC learnability. We define an example as a pair $(I,a)$ of database $I$ and individual $a\in \adom(I)$. A concept $C$ fits $(P,N)$ if $a\in C^I$ for all $(I,a)\in P$ and $a\notin C^I$ for all $(I,a)\in N$. In addition we require the definition of the error of two concepts as follows. Given a probability distribution $\mathbb{P}$ on examples and concepts $C$ and $D$ the error $\text{error}_{\mathbb P}(C,D)$ from $C$ to $D$ w.r.t. $\mathbb P$ is the probability to draw an example $(I,a)$ with $a\in C^I\Delta D^I$, that is \[\text{error}_{\mathbb P}(C_T,C) = \mathbb P(\{(I,a)\mid a\in C^I\Delta D^I\}).\]

\begin{definition}
    A sample-efficient PAC learning algorithm is a fitting algorithm associated with a polynomial $m: (0,1) \times (0,1)\times\mathbb N\rightarrow \mathbb N$ so that for all finite signatures $\Sigma$, probability distributions $\mathbb P$ on examples over $\Sigma$, $\varepsilon,\delta\in (0,1)$, $s\in\mathbb N$, and concepts $C_T$ with $\lVert C_T\rVert\le s$: if $\mathbf{A}$ is run on a sample $P,N$ of $m(\frac 1\delta,\frac 1\varepsilon,s)$ examples drawn according to $\mathbb P$ then it returns with probability at least $1-\delta$ a concept $C$ that fits $P,N$ and satisfies $\text{error}_{\mathbb P}(C_T,C)\le \varepsilon$.
\end{definition}
\begin{table*}[!ht]
\tiny
\begin{tabular}{lcccccccc}
\toprule
& \tiny Carcinogenesis & \tiny Hepatitis & \tiny Lymphography & \tiny Mammographic & \tiny Mutagenesis & \tiny Nctrer & Premierleague & Pyrimidine\\ \midrule
\multirow{3}{*}{$\top$} &  
$0.55 \pm  0.00$ & 
$0.41 \pm  0.00$ & 
$0.57 \pm  0.00$ & 
$0.46 \pm  0.00$ & 
$0.33 \pm  0.00$ & 
$0.59 \pm  0.00$ & 
$0.50 \pm  0.00$ & 
$0.50 \pm  0.00$ 
\\
& $1.00 \pm  0.00$ & 
$1.00 \pm  0.00$ & 
$1.00 \pm  0.00$ & 
$1.00 \pm  0.00$ & 
$1.00 \pm  0.00$ & 
$1.00 \pm  0.00$ & 
$1.00 \pm  0.00$ & 
$1.00 \pm  0.00$ 
\\
& $0.71 \pm  0.00$ & 
$0.58 \pm  0.00$ & 
$0.73 \pm  0.00$ & 
$0.63 \pm  0.00$ & 
$0.50 \pm  0.00$ & 
$0.74 \pm  0.00$ & 
$0.67 \pm  0.00$ & 
$0.67 \pm  0.00$ 

\\ \midrule
\multirow{3}{*}{EvoLearner} &
$0.56 \pm  0.03$ & 
$0.77 \pm  0.10$ & 
$0.82 \pm  0.09$ & 
$0.80 \pm  0.05$ & 
$1.00 \pm  0.00$ & 
$1.00 \pm  0.00$ & 
$1.00 \pm  0.00$ & 
$0.93 \pm  0.12$ 
\\
& $5.50 \pm  1.51$ & 
$4.50 \pm  1.35$ & 
$17.40 \pm  9.03$ & 
$3.20 \pm  1.48$ & 
$3.00 \pm  0.00$ & 
$2.40 \pm  0.97$ & 
$5.00 \pm  0.94$ & 
$6.40 \pm  1.26$ 
\\
& $0.71 \pm  0.01$ & 
$0.77 \pm  0.07$ & 
$0.84 \pm  0.08$ & 
$0.78 \pm  0.05$ & 
$1.00 \pm  0.00$ & 
$1.00 \pm  0.00$ & 
$1.00 \pm  0.00$ & 
$0.91 \pm  0.14$ 
\\
\midrule
\multirow{3}{*}{TDL} &
$0.58 \pm  0.09$ & 
$0.84 \pm  0.05$ & 
$0.80 \pm  0.12$ & 
$0.60 \pm  0.05$ & 
$0.86 \pm  0.15$ & 
$1.00 \pm  0.00$ & 
$0.68 \pm  0.17$ & 
$0.82 \pm  0.26$ 
\\
&$879.60 \pm  132.97$ & 
$1436.80 \pm  46.49$ & 
$112.10 \pm  23.33$ & 
$5780.70 \pm  302.98$ & 
$26.50 \pm  8.92$ & 
$2.00 \pm  0.00$ & 
$298.40 \pm  35.22$ & 
$15.90 \pm  3.78$ 
\\
&$0.62 \pm  0.07$ & 
$0.83 \pm  0.05$ & 
$0.82 \pm  0.11$ & 
$0.69 \pm  0.03$ & 
$0.78 \pm  0.31$ & 
$1.00 \pm  0.00$ & 
$0.47 \pm  0.37$ & 
$0.82 \pm  0.24$ 
\\
\midrule
\multirow{3}{*}{Algorithm \ref{alg:approximatefitting}} & 
$0.59 \pm  0.06$ & 
$0.62 \pm  0.05$ & 
$0.73 \pm  0.13$ & 
$0.78 \pm  0.04$ & 
$1.00 \pm  0.00$ & 
$0.99 \pm  0.02$ & 
$1.00 \pm  0.00$ &
 $0.93 \pm  0.12$ 
\\
&$4715.60 \pm  455.32$ & 
$53269.80 \pm  1549.82$ & 
$330.80 \pm  18.98$ & 
$1925.40 \pm  72.76$ & 
$1.00 \pm  0.00$ & 
$64.20 \pm  5.35$ & 
$23.40 \pm  1.26$ & 
$60.80 \pm  3.94$ 
\\
&$0.56 \pm  0.06$ & 
$0.53 \pm  0.06$ & 
$0.74 \pm  0.13$ & 
$0.76 \pm  0.05$ & 
$1.00 \pm  0.00$ & 
$0.99 \pm  0.02$ & 
$1.00 \pm  0.00$ & 
$0.90 \pm  0.16$ 
\\\midrule
\multirow{3}{*}{Our Tool} & 
$0.70 \pm  0.10$ & 
$0.84 \pm  0.03$ & 
$0.80 \pm  0.09$ & 
$0.83 \pm  0.04$ & 
$0.97 \pm  0.11$ & 
$1.00 \pm  0.00$ & 
$0.97 \pm  0.05$ & 
$0.95 \pm  0.11$
\\
&$6.00 \pm  0.00$ & 
$5.00 \pm  0.00$ & 
$10.00 \pm  0.00$ & 
$9.10 \pm  0.32$ & 
$1.00 \pm  0.00$ & 
$2.00 \pm  0.00$ & 
$2.00 \pm  0.00$ & 
$5.90 \pm  0.32$ 
\\
&$0.73 \pm  0.10$ & 
$0.81 \pm  0.03$ & 
$0.83 \pm  0.08$ & 
$0.81 \pm  0.05$ & 
$0.90 \pm  0.32$ & 
$1.00 \pm  0.00$ & 
$0.97 \pm  0.05$ & 
$0.95 \pm  0.12$ 
\\

\bottomrule
\end{tabular}
\caption{Accuracies (first line), concept sizes (second line) and F1-scores (third line) of concept learning implementations on the SML-benchmarks. Timeout of 5min}
\label{tab:generalization_full}
\end{table*}
\begin{theorem}
    Algorithm $\mathbf{A}_f$ is not a sample-efficient PAC learning Algorithm.
\end{theorem}
\begin{proof}
Assume that Algorithm $\mathbf{A}_f$ is a sample-efficient PAC learning algorithm with associated polynomial $m$. Let $\Sigma = \{r,f\}$ with, $\dom(f) = \mathbb N$, $C_T=\exists t^s.\top$, $\delta = $ $\varepsilon = \frac 12$. Let $s$ be large enough, so that $m(\frac 1\varepsilon,\frac 1\delta,s)\le 2^{s-1}$. We construct sets of positive and negative examples to be assigned non-zero probability and show that if $\mathbf A_f$ is started on a sample of $m(\frac 1\varepsilon,\frac 1\delta,s)$ examples, then with probability $1-\delta$ it returns a concept $C$ with $\text{error}_{\mathbb P}(C_T,C) > \varepsilon$. To define our examples, let $g: \{2i\mid 1\le i\le 2^{2^s}\}\rightarrow 2^{\{1,\dots,2^s\}}$ be any bijective function.  Consider the set of positive examples $P=\{(\mathcal I_1,a_1),\dots, (\mathcal I_{2^s}, a_{2^s})\}$ where each example $(I_i,a_i)$ is obtained by adding
\begin{enumerate}
    \item a $t^s$-path outgoing from $a_i$,
    \item for all $j\in \{1,\dots,2^{2^s+1}\}$ so that $i\in g(j)$ and $|g(j)|\le2^{s-1}$ an $r$-edge from $a_i$ to an element $a_j$ with $f(a_j,j)\in I_i$ and
\end{enumerate}    
The set of negative examples is obtained as  $N= \{(J,b)\}$ where for each $j\in\{2i-1\mid 0\le i\le 2^{2^s}+2\}$ the database $J$ contains an $r$-edge to an individual $b_j$ with $f(b_j,j)\in J$. We define $\mathbb P$ to be the probability distribution that assigns $\mathbb P(I,a) = \frac{1}{2^{s+1}}$ to each positive example $(I,a)$ and $\mathbb P(J,b) = \frac 12$. Now, let $(P',N')$ be sets of positive and negative examples obtained from a sample of $m(\frac 1\varepsilon,\frac 1\delta,s)$ examples drawn according to $\mathbb P$. We have $|P'|< 2^{s-1}$ because of $m(\frac 1\varepsilon,\frac 1\delta,s)<2^{s-1}$. We may assume $(J,b)\in N'$ as the probability that this is the case is $1-(\frac1{2^{s+1}})^s\ge \frac 12=\delta$. Observe that the following holds for any such $(P',N')$.
\begin{enumerate}
    \item There is no $j$ so that the concept $\exists r.(f\bowtie j)$ with $\bowtie\in\{\le,\ge\}$  fits $P',N'$.
    \item The concept $\exists r. ((f \le j)\sqcap (f\ge j))$ where $j$ is uniquely determined by $g(j) = \{i\mid(I_i,a_i)\in P'\}$ fits $P',N'$.
    \item For any concept of the form $C=\exists r. ((f \le j)\sqcap (f\ge j))$ that fits $P',N'$:
    \[|\{(I,a)\in P\mid a\notin C^I\}|\ge 2^{{s-1}}\]
\end{enumerate}
It follows from Points 1. and 2. that if $\mathbf{A}_f$ is started on $P',N'$ it returns a concept of the form $C=\exists r. ((f \le j)\sqcap (f\ge j))$. Point 3. then implies that $\text{error}_{\mathbb P}(C_T,C) > \frac 12$.
\end{proof}

\begin{table}[!t]
\centering
\footnotesize
\begin{tabular}{rrrr}
\toprule
$w$ & \scriptsize Mammographic & \scriptsize Carcinogenesis  & \scriptsize Lymphography \\\midrule
 1 & 21.5 & 341.5 & 51.8 \\
 2 & 11.2 & 173 & 30.8 \\
 4 & 6.8 & 113.3 & 21.5 \\
8  & 5.8 & 99.6 & 19.6\\
\bottomrule
\end{tabular}
\caption{Running times of bounded fitting up to size 8 with number of worker threads $w$. Average of 3 runs.}
\label{table:exp-parallel}
\end{table}
\thmoccamconstantoutdegree*
\begin{proof}
    We begin with the second part, that $\mathbf{A}_f$ is an Occam Algorithm if input databases have bounded domains. If there is a bounded number of values for each feature then the number of new concept names that the reduction of from Lemme \ref{lem:reduction} introduces is bounded in the same way. Thus we deal with a finite number of concept names. It then follows from the proof of Theorem 4 of \cite{DBLP:conf/ijcai/CateFJL23} that Algorithm $\mathbf A_f$ is an Occam algorithm.

    For the first part of the statement, let $\mathbf{A}_{f,\ell}$ be Algorithm $A_{f}$ obtained by bounding outdegree in interpretations by $\ell$. We require the following two definitions. For a set of examples $E$ let $\text{reach-val}^f_s(E)$ be the set of all values for $f$ appearing in $E$ at depth at most $s$. $|\text{reach-val}^f_s(E)|\le |E| \ell^s$. Let $\text{val}_f(C)$ be the set of all values for $f$ appearing in $C$.    
    \bigskip
    
    Let $\Sigma$ be any signature, $s,m\ge 1$ and $C_T$ be any target concept with $\lVert C\rVert\le s$. For a set of examples $E$ let $\mathcal H^{\mathbf A_{f,\ell}}_E(\Sigma,s,m)$  be the set of all concepts from $\mathcal H^{\mathbf A_{f,\ell}}(\Sigma,s,m)$ that only use feature values $f\bowtie v$ with $v\in  \text{reach-val}^f_s(E)$. Since any concept $C\in\mathcal H^{\mathbf A_{f,\ell}}(\Sigma,s,m)$ has size at most $\lVert C\rVert\le s$, it follows that 
    \begin{equation}\label{eq:actual_hyp_space_bound}
    |\mathcal H_E^{\mathbf A_{f,\ell}}(\Sigma,s,m)|\le (|\Sigma|+c+|\Sigma|l^sm)^s
    \end{equation}
    for some constant $c$ that accounts for operators. Now, let $d$ be the VC-dimension of $\mathcal H^{\mathbf A_{f,\ell}}(\Sigma,s,m)$ and let $E$ with $|E|=d$ and $\mathcal C\subseteq \mathcal H^{\mathbf A_{f,\ell}}(\Sigma,s,m)$ with $|\mathcal C|=2^d$ w.l.o.g. be a set of examples and a set of concepts, respectively, such that $E$ is shattered by $\mathcal C$. Because of $|\mathcal C|=2^d$ we have that for any $C_1,C_2\in \mathcal C$ with $C_1\ne C_2$: \begin{equation}\label{eq:ext_diff}
        C_1^I\cap E\ne C_2^I\cap E.
    \end{equation}
    
    \textit{Claim:} There is a set $\mathcal C'\subseteq \mathcal H_E^{\mathbf A_{f,\ell}}(\Sigma,s,m)$ with $|\mathcal C'|=2^d$ that shatters $E$.
    
    \bigskip
    
    For $C\in\mathcal C$, let $C'$ be $C$ with all subconcepts $f\bowtie v$ of $C$ with $f\in\Sigma_F,\bowtie\in\{\le,\ge\},v\notin \text{reach-val}^f_s(E)$ replaced by $f\bowtie v'$ where in case of $\bowtie = \le, v'$ is the largest element of $\text{reach-val}^f_s(E)$ with $v'\le v$ and in case of $\bowtie = \ge, v'$ is the smallest element of $\text{reach-val}^f_s(E)$ with $v'\ge v$.
    Define $g:\mathcal C\rightarrow \mathcal H_E^{\mathbf A_{f,\ell}}(\Sigma,s,m)$ by $g(C)= C'$. Then, \begin{enumerate}
        \item $g(\mathcal C)\subseteq\mathcal H_E^{\mathbf A_{f,\ell}}(\Sigma,s,m) $, that is, $g$ is well-defined. This follows immediately from the construction of $C'$ as in $C'$ all values for feature restrictions appear in $E$.
        \item For any $a\in E: a\in C^I\Leftrightarrow a\in C'^I$. This follows from the construction of $C'$.
        \item $\mathcal C' = g(\mathcal C)$ shatters $E$, an immediate consequence of Point 2 and that $\mathcal C$ shatters $E$.
        \item $g$ is injective and consequently $|\mathcal C'| = 2^d$.
    \end{enumerate}    
    To see that Point 4. holds, assume $g$ is not injective, that is, there are $C_1,C_2, D$ with $g(C_1) = g(C_2) = D$. Point 2. then implies $C_1^I\cap E = C_2^I \cap E$, a contradiction to Equation \eqref{eq:ext_diff}.

   From the Claim and Equation \eqref{eq:actual_hyp_space_bound} we obtain 
    \[2^d\le | \mathcal H_E^{\mathbf A_{f,\ell}}(\Sigma,s,m)| \le  (|\Sigma|+c+|\Sigma|l^sm)^s\] and thus
    \begin{align*}
        &d \le \log((|\Sigma|+c+|\Sigma|l^sm)^s)\\
        & = s\log(|\Sigma|+c+|\Sigma|l^sm)\\
        & \le s\log(2|\Sigma|cl^sm)\\
        & = s(\log(2|\Sigma|c)+\log(l^s)+\log(m)))\\
        & \le |\Sigma|c's^2\sqrt m\tag{for an appropriate constant $c'$} 
    \end{align*}
\end{proof}

\section{Additional Experiment Data}
We give the remaining results for generalization on the SML-Benchmarks that were omitted in the main part in Table \ref{tab:generalization_full}. Further results regarding parallelization are given in Table \ref{table:exp-parallel}.

\end{document}